\documentclass{article}

\PassOptionsToPackage{numbers}{natbib}

\usepackage[main,final]{neurips_2026}

\usepackage[utf8]{inputenc} 
\usepackage[T1]{fontenc}    
\usepackage{hyperref}       
\usepackage{url}            
\usepackage{booktabs}       
\usepackage{amsfonts}       
\usepackage{nicefrac}       
\usepackage{microtype}      
\usepackage{xcolor}         

\usepackage{amsmath}
\usepackage{amsfonts}
\usepackage{amssymb}
\usepackage{amsthm}
\usepackage{mathtools}
\usepackage{dsfont}

\usepackage[capitalize,noabbrev]{cleveref}
\crefname{figure}{Figure}{Figures}
\crefname{equation}{}{}

\usepackage{wrapfig,lipsum,booktabs}

\usepackage{thm-restate}

\makeatletter
\newcommand{\IfRestatedTF}[2]{\ifthmt@thisistheone #2\else #1\fi}
\makeatother

\theoremstyle{plain}
\newtheorem{theorem}{Theorem}[section]

\newtheorem{lemma}[theorem]{Lemma}

\theoremstyle{definition}

\theoremstyle{remark}

\DeclareMathOperator*{\argmax}{arg\,max}

\newcommand{\E}{\mathbb{E}}

\usepackage{graphicx}
\graphicspath{{./img/}}
\newcommand{\Sim}{\mathrm{sim}}
\newcommand{\Tan}{\mathrm{Tan}}
\newcommand{\Cos}{\mathrm{Cos}}
\newcommand{\Bit}{\mathrm{Bit}}
\newcommand{\bx}{\boldsymbol{x}}
\newcommand{\bX}{\boldsymbol{X}}
\newcommand{\by}{\boldsymbol{y}}
\newcommand{\bY}{\boldsymbol{Y}}
\newcommand{\bc}{\boldsymbol{c}}
\newcommand{\HRAtOne}{{\mathrm{HR}@1}}
\newcommand{\HRAtk}{{\mathrm{HR}@k}}
\newcommand{\indicator}{\mathds{1}}

\providecommand{\ul}[1]{\underline{#1}}  
\newcommand{\z}{\phantom{0}}                       
\newcommand{\res}[2]{#1{\tiny\textpm#2}}           
\newcommand{\best}[2]{\textbf{#1}{\tiny\textbf{\textpm#2}}}
\newcommand{\second}[2]{\ul{#1}{\tiny\textpm#2}}

\usepackage{verbatim} 
\usepackage{float}

\title{On objective mismatch in molecular retrieval from tandem mass spectrometry}

\author{%
    Gaetan De Waele \\
   University of Antwerp /\\Ghent University\\
  \texttt{gaetan.dewaele@uantwerpen.be}
  \And
  Marek Wydmuch \\
  Snowflake AI Research / \\
  Poznan University of Technology \\
  \texttt{marek.wydmuch@snowflake.com}
  \And
  Krzysztof Dembczyński \\
  Yahoo Research / \\
  Poznan University of Technology \\
  \texttt{kdembczynski@cs.put.poznan.pl}
  \And
  Wojciech Kotłowski \\
  Poznan University of Technology \\
  \texttt{wkotlowski@cs.put.poznan.pl}
  \And
  Willem Waegeman \\
  Ghent University \\
  \texttt{willem.waegeman@ugent.be}
}

\begin{document}

\maketitle

\begin{abstract}
Identifying molecular structures from LC-MS/MS spectra is a central problem in computational metabolomics, often approached by predicting molecular fingerprints and retrieving candidates via similarity search.
Despite widespread use, the relationship between training objectives for fingerprint prediction and downstream retrieval performance remains poorly understood.
In this work, we investigate the alignment between these objectives and show that they can be misaligned under uncertainty.
Adopting a decision-theoretic perspective, we derive novel regret bounds that characterize when Bayes-optimal predictors for fingerprint similarity diverge from those for molecular retrieval.
Our analysis reveals that optimizing standard similarity-based losses can degrade retrieval performance, and that the extent of this mismatch depends on the similarity structure of candidate sets.
Empirically, we validate our theory on the MassSpecGym and Enveda-180 datasets, observing a Pareto frontier between fingerprint accuracy and retrieval metrics across commonly used loss functions.
These findings highlight potential trade-offs in fingerprint-based molecular identification and provide principled guidance for the design of learning objectives in small molecule mass spectrometry.
\end{abstract}

\section{Introduction}

Liquid chromatography-tandem mass spectrometry (LC-MS/MS) is a standard analytical approach for identifying small molecules in biological samples (i.e., metabolites)~\citep{de2022good}: compounds are separated (LC), ionized and measured by mass (MS$^1$), and then fragmented (MS$^2$).
Each precursor yields a spectrum of peaks defined by intensity and mass-over-charge (\textit{m/z}).
Characterizing the underlying compound from these fragmentation patterns is a central problem in computational mass spectrometry~\citep{stravs2024decade}, complicated by the diversity of chemical structures, the variability of fragmentation behavior, and the limited coverage of reference spectral databases.
Existing approaches fall into three categories~\citep{butler2023ms2mol}.
\textit{Spectral library matching} compares the query spectrum to a reference database of molecule-spectrum pairs~\citep{huber2021ms2deepscore}.
\textit{Molecular retrieval} predicts a molecular property (vector) that is then used to index into a molecular database~\citep{duhrkop2015searching}.
\textit{De novo structure elucidation} predicts the molecular structure directly, via a structured output or generative modeling approach~\citep{stravs2022msnovelist, butler2023ms2mol, bohde2025diffms, han2025ms, bohde2026frigid}.
Here, we focus on retrieval from a candidate database.
A paradigmatic approach to this problem is to predict molecular fingerprints: sparse (binary) vectors encoding chemical substructures, which are then compared to candidate molecules via similarity functions such as Tanimoto or cosine similarity~\citep{cereto2015molecular, bajusz2015tanimoto}.

\paragraph{Related work.}
Early fingerprint predictors (CSI:FingerID~\citep{duhrkop2015searching}) optimize bitwise classification, while neural methods such as MIST~\citep{goldman2023annotating} adopt cosine-loss pretraining followed by contrastive fine-tuning. More recent works (JESTR~\citep{kalia2025jestr}, CMSSP~\citep{chen2024cmssp}, MVP~\citep{zhou2026learning}) train directly with contrastive or metric-learning objectives, often bypassing explicit fingerprint prediction by embedding spectra and molecules in a joint latent space.
{
\parfillskip=0pt
\parskip=0pt
\par}
\begin{wrapfigure}{r}{6.5cm}
    \begin{center}
	\includegraphics[width=0.92\linewidth]{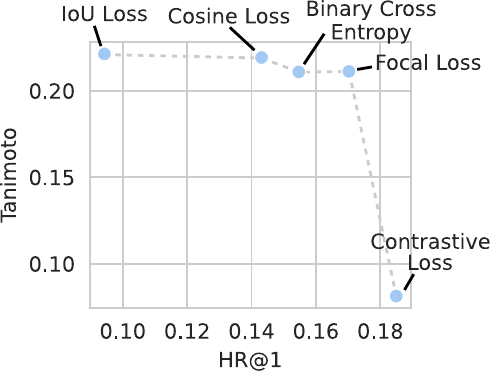}
    \end{center}
	\caption{
		Empirically, different loss functions form a Pareto front trading off fingerprint similarity (y-axis) and retrieval performance (in terms of hit rate @ 1, x-axis) on MassSpecGym. See \cref{sec:msmolemainres} for experimental details.
	}
    \vspace{-0.5cm}
	\label{fig:inverse_corr}
\end{wrapfigure}
Orthogonal to retrieval, \textit{de novo} methods~\citep{bohde2025diffms, han2025ms, bohde2026frigid} generate structures without a candidate set and lie outside the scope of this work.
An extended discussion is given in \cref{app:related_work}.
Decision-theoretic and surrogate-loss analyses~\citep{dembczynski2012label, bartlett2006convexity, waegeman2014, yang2020consistency} characterize when training losses are calibrated for a target metric, but these tools have not been applied to molecular retrieval.
In metabolomics, there is limited understanding of the trade-offs involved in choosing a loss function for fingerprint-based retrieval.

In the present study, we establish a fundamental trade-off between two objectives that are often conflated: (1) predicting accurate fingerprints, and (2) retrieving the correct molecule.
In other words: the "best" fingerprint prediction is not necessarily the most discriminative one.
We first explain this phenomenon theoretically: adopting a decision-theoretic perspective, we derive novel regret bounds characterizing when Bayes-optimal decisions for similarity metrics must diverge from those for retrieval metrics.
Our analysis reveals that the severity of this trade-off depends on the `similarity band' of candidate sets, a quantity computable from data.
We then confirm the trade-off empirically through systematic experiments on the MassSpecGym~\citep{bushuiev2024massspecgym} and Enveda-180~\citep{krettler2026enveda} datasets, showing that different loss functions form a Pareto front (\cref{fig:inverse_corr}).
Together, these results deliver practical guidelines for loss function selection in molecular retrieval.

\section{Problem setting and state-of-the-art}

\subsection{Fingerprint-based molecular retrieval} \label{sec:eval_fp}

Let $\mathcal{X}$ denote the space of mass spectra and $\mathcal{Y}\subset\{0,1\}^m$ the space of $m$-dimensional molecular fingerprints.
Let us denote a dataset of mass spectrum--molecule pairs by $\mathcal{D} = \{(\boldsymbol{x}^{(i)}, \boldsymbol{y}^{(i)})\}_{i=1}^n$, with $\boldsymbol{x}^{(i)}\in\mathcal{X}$ and $\boldsymbol{y}^{(i)}\in\mathcal{Y}$.
For each input $i$, one can define a set of retrieval candidate fingerprints $\mathcal{C}^{(i)}=\{\boldsymbol{c}_1^{(i)}, \dots, \boldsymbol{c}_{l_i}^{(i)}\}\subseteq\mathcal{Y}$, where $l_i=|\mathcal{C}^{(i)}|$ is the number of candidates for input $i$.
As mass spectrometry measures the mass of the unknown ion, and chemical formulae can be inferred with reasonable accuracy~\citep{bushuiev2024massspecgym}, retrieval candidates can be drawn from a database of compounds based on either mass or formula.
We adopt the closed-world assumption that the ground-truth fingerprint belongs to the candidate set, $\boldsymbol{y}^{(i)}\in\mathcal{C}^{(i)}$, as is standard in benchmark settings.
We use $\boldsymbol{c}_j^{(i)}$ to enumerate candidate fingerprints and $\boldsymbol{y}^{(i)}$ to denote the observed ground-truth fingerprint.

We consider models that predict fingerprints $\hat{\boldsymbol{y}}^{(i)}\in[0,1]^m$.
Let us denote the fingerprint predictor as $f$, producing fingerprint predictions according to $
\hat{\boldsymbol{y}}^{(i)} = f(\boldsymbol{x}^{(i)}) = \sigma(\boldsymbol{W}_\text{out} E(\boldsymbol{x}^{(i)}))$,  
where all but the last layers are included in the embedder $E$, which can be any neural network function embedding the spectra to a $d$-dimensional space.
A last linear layer, parameterized by a learnable $\boldsymbol{W}_\text{out} \in \mathbb{R}^{m \times d}$, then projects the embedding to fingerprint space, upon which a sigmoid operation $\sigma$ is applied (\cref{fig:msmolemain}A).

Performance is typically assessed via (1) fingerprint accuracy and (2) retrieval of the ground-truth molecule from candidates~\citep{bushuiev2024massspecgym, goldman2023annotating}.
Fingerprint accuracy is usually assessed through the Tanimoto similarity~\citep{bajusz2015tanimoto}, also known as the Jaccard index and the intersection-over-union (IoU):
\begin{equation*}
U_\Tan \big(\boldsymbol{y}^{(i)},\boldsymbol{\hat{y}}^{(i)} \big)
    =  \frac{\hat{\boldsymbol{y}}^{(i)\top} \boldsymbol{y}^{(i)}}{\|\hat{\boldsymbol{y}}^{(i)}\|_1 + \|\boldsymbol{y}^{(i)}\|_1 - \hat{\boldsymbol{y}}^{(i)\top} \boldsymbol{y}^{(i)}},
\end{equation*}
where $\|\cdot\|_1$ denotes the L1 norm.  While the formula does not strictly require it, Tanimoto similarities are usually computed over two binarized vectors (usually obtained by thresholding)~\citep{goldman2023annotating}.

\begin{figure*}[t!]
	\centering
	\includegraphics[width=1.00\linewidth]{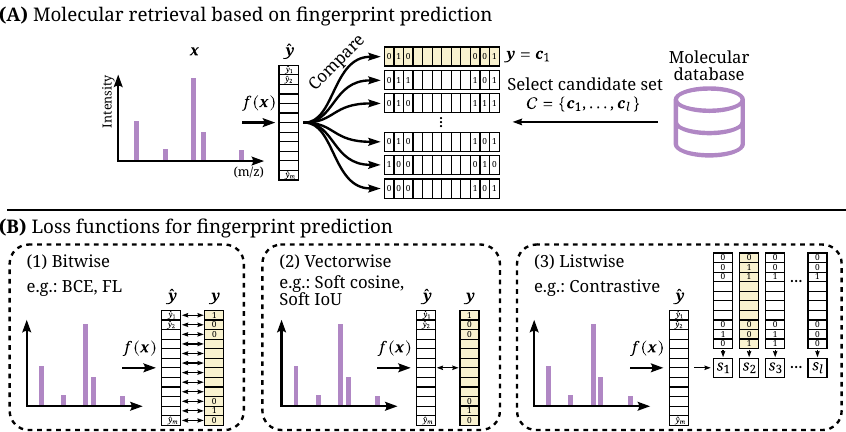}
	\caption{
		\textbf{(A)} General framework used in this study.
		In this figure, all superscript notations indicating the sample index $(i)$ are omitted.
		A model $f$ predicts a fingerprint $\hat{\boldsymbol{y}}^{(i)}\in[0,1]^m$ from an input spectrum $\boldsymbol{x}^{(i)}$.
		The prediction is then compared to a set of fingerprints $\mathcal{C}^{(i)}$ derived from candidate molecules in a database.
		\textbf{(B)} Loss functions for training $f(\cdot)$ generally fall into one of three categories: (1) bitwise losses (e.g., binary cross entropy or binary focal loss), (2) vectorwise losses (e.g., soft cosine loss or soft IoU loss), or (3) listwise losses (e.g., contrastive loss).
	}
	\label{fig:msmolemain}
\end{figure*}

To assess retrieval performance, we use the Hit Rate @ $k$ (HR$@k$)~\citep{bushuiev2024massspecgym}.
To compute this score, a set of matching scores $\mathcal{S}^{(i)}$ is first computed: $ \mathcal{S}^{(i)} = \{s_1^{(i)}, \dots, s_{l_i}^{(i)}\}$, where
$s_j^{(i)} = \mathrm{sim}(\boldsymbol{c}_j^{(i)} ,\hat{\boldsymbol{y}}^{(i)})$, for all  $j\in\{1, \dots, l_i\}$,
with $\text{sim}(\cdot, \cdot)$ usually the Tanimoto or cosine similarity.
This set quantifies a "matching score" of the predicted fingerprint against all the candidates in the set $\mathcal{C}^{(i)}$.
Let $s_{\boldsymbol{y}^{(i)}}$ denote the matching score to the true molecule $\boldsymbol{y}^{(i)}$.
Then, let us define the set of $k$ highest scores as $\mathcal{S}_k^{(i)}$, where $\mathcal{S}_k^{(i)} \subseteq \mathcal{S}^{(i)}$.
For a test instance indexed by $i$, the HR$@k$ utility is then computed as:
\begin{equation*}
	U_\HRAtk \big(\boldsymbol{y}^{(i)}, \mathcal{S}_k^{(i)} \big) = \indicator\left[s_{\boldsymbol{y}^{(i)}} \in \mathcal{S}^{(i)}_k\right],
\end{equation*}
where $\indicator$ is the indicator function.

\subsection{Loss functions for fingerprint prediction} \label{sec:loss_fp}

Broadly speaking, one can categorize the loss functions previously employed into one of three categories, here called: (1) bitwise, (2) vectorwise, and (3) listwise (\cref{fig:msmolemain}B).
In the following, all loss functions are assumed to receive soft predictions $\hat{\boldsymbol{y}}^{(i)}\in[0,1]^m$.

\paragraph{Bitwise losses.}
The simplest way to approach the problem of fingerprint prediction is through the lens of multi-label classification.
In this study, we refer to "bitwise losses" as all losses that seek to optimize the predictions of individual bits using label-wise decomposable functions~\citep{dembczynski2012label}, for example, the binary cross entropy (BCE):
\begin{equation}
\label{eq:bceloss}
\text{BCE}\big(\boldsymbol{y}^{(i)}, \hat{\boldsymbol{y}}^{(i)} \big) =
-\sum_{k=1}^m \left[ y_k^{(i)} \log(\hat{y}_k^{(i)}) + (1-y_k^{(i)})\log(1-\hat{y}^{(i)}_k) \right].
\end{equation}
The focal loss (FL)~\citep{lin2017focal} is a modification of the BCE that down-weights easy bits so the model focuses more on learning from misclassified bits:
\begin{equation}
\label{eq:flloss}
\text{FL} \big(\boldsymbol{y}^{(i)}, \hat{\boldsymbol{y}}^{(i)} \big) = -\sum_{k=1}^m \Big[ y_k^{(i)} (1-\hat{y}_k^{(i)})^\gamma \log(\hat{y}_k^{(i)})
+ (1-y_k^{(i)})(\hat{y}^{(i)}_k)^\gamma \log(1-\hat{y}^{(i)}_k) \Big],
\end{equation}
with $\gamma$ a hyperparameter.
By down-weighting majority-class bits, FL is a natural choice for sparse fingerprints.

\paragraph{Vectorwise losses.}
As a second category, we consider loss functions that optimize the entire predicted label vector at once, i.e., are not label-wise decomposable.
Molecular fingerprint predictions are commonly optimized using the soft cosine similarity loss~\citep{goldman2023annotating, bushuiev2024massspecgym, dreams}:
\begin{equation}
\label{eq:cosine_loss}
	\text{CosineLoss} \big(\boldsymbol{y}^{(i)}, \hat{\boldsymbol{y}}^{(i)} \big) = 1 - \frac{\hat{\boldsymbol{y}}^{(i)\top} \boldsymbol{y}^{(i)}}{\|\hat{\boldsymbol{y}}^{(i)}\|_2 \cdot \|\boldsymbol{y}^{(i)}\|_2}.
\end{equation}
Alternatively, the Tanimoto similarity can be directly maximized~\citep{nowozin2014optimal, wang2023jaccard}. This is equivalent to minimizing the Tanimoto loss or IoU loss.
In this paper, we will use the notion Tanimoto-similarity when
maximizing a utility, and the IoU loss when minimizing a loss.

\paragraph{Listwise losses.}
A last category of loss functions attempts to learn how an input spectrum relates to every putative structure in the candidate list $\mathcal{C}^{(i)}$, here referred to as "listwise" losses.
A prevalent example in this category is the family of contrastive losses~\citep{chen2020simple}.
In metabolomics, they are usually formulated along the line of:
\begin{equation}
\label{eq:contrastiveloss}
	\text{Contr}\big(\boldsymbol{y}^{(i)},\mathcal{C}^{(i)} \big) = -\log \frac{e^{g(\boldsymbol{x}^{(i)}, \boldsymbol{y}^{(i)})/\tau}}{\sum_{\boldsymbol{c}^{(i)}_j \in \mathcal{C}^{(i)}} e^{g(\boldsymbol{x}^{(i)}, \boldsymbol{c}^{(i)}_j)/\tau}},
\end{equation}
where $\boldsymbol{y}^{(i)}$ is the true candidate, $\tau$ is a temperature hyperparameter, and $g$ is any function that returns a scalar score quantifying the match between spectrum (embedding) and candidate fingerprints~\citep{goldman2023annotating, kalia2025jestr, chen2024cmssp}.
The above loss function essentially converts the multi-label fingerprint prediction problem into a multi-class classification problem, where the singular positive class is the correct fingerprint and all other candidate label bitvectors are taken as negative classes.
As it is computationally infeasible to enumerate a softmax over all possible label vectors $\boldsymbol{y}\in\{0,1\}^m$, the usage of the candidate set $\mathcal{C}^{(i)}$, with $|\mathcal{C}^{(i)}| \ll 2^m$, can be seen as some form of deterministic negative sampling~\citep{mikolov2013distributed}.

\section{Theoretical analysis of loss functions}

A natural assumption in fingerprint-based retrieval is that better fingerprint predictions yield better retrieval: if a model's output is closer to the ground-truth fingerprint under some similarity, it should also rank the ground-truth molecule higher among candidates.
This section shows that the assumption fails in general.
Concretely, we ask: when must a predictor that is optimal for a similarity utility (e.g., Tanimoto, cosine) be suboptimal for retrieval (HR@k), and how large can the gap be?
We answer both questions in the decision-theoretic (plug-in) framework~\citep{devroye2013probabilistic, bartlett2006convexity}, which lets us compare losses through the metric each one implicitly optimizes at training time.\footnotemark{}

\footnotetext{
Two frameworks exist for analyzing losses in structured prediction.
The plug-in (decision-theoretic) view fits a probabilistic model and decodes by maximizing the target metric's expected value, allowing direct regret analysis when one metric is optimized in place of another.
The surrogate-minimization view trains directly against a differentiable relaxation of a non-decomposable task loss (e.g., structured SVMs \citep{Hazan2010}, or the convex Lovász extension of IoU yielding the Lovász-Softmax surrogate \citep{Berman2017TheLL}), and asks when the surrogate is consistent for the target.
The vectorwise losses studied here (sigmoid-smoothed cosine, soft IoU) are themselves surrogates of this kind.
The two views are asymptotically equivalent under standard conditions but differ in finite-sample and misspecified regimes \citep{Ye2012, dembczynski13}.
We adopt the plug-in view because our object of study is a comparison of trained models: it provides a single decision rule under which different losses can be ranked by the metric each implicitly optimizes, which surrogate analysis does not.
}

\subsection{Bayes-risk analysis} \label{sec:contrastive_as_retrieval}


To perform Bayes-risk analyses, we assume that training and test data is i.i.d.\ sampled from an unknown probability distribution $P$ over $\mathcal{X} \times \mathcal{Y}$.
The associated ground-truth conditional distribution will be denoted by $p (\boldsymbol{y} \, | \, \boldsymbol{x})$.
We use $\boldsymbol{Y}$ for the random target fingerprint and $\boldsymbol{y}$ for a possible realization, while $\boldsymbol{c}_j$ denotes a candidate fingerprint.
For a fixed spectrum, we omit the observation index and write $\mathcal{C}=\{\boldsymbol{c}_1,\ldots,\boldsymbol{c}_l\}$.
Under the closed-world assumption, $p(\boldsymbol{y}\mid\boldsymbol{x})=0$ for $\boldsymbol{y}\notin\mathcal{C}$, so expectations over the target fingerprint can be written as sums over candidates.
For a spectrum $\boldsymbol{x}$ and performance metric $U: \mathcal{Y} \times \mathcal{Y} \rightarrow \mathbb{R}$, the prediction that maximizes the expected value of the metric $\boldsymbol{y}^\star_U$ (i.e., the Bayes-optimal decision) is given by:
\begin{equation}
\label{eq:bayes_opt}
    \boldsymbol{y}^\star_U(\boldsymbol{x}) =  \argmax_{\hat{\boldsymbol{y}} \in \{0,1\}^m} \, \sum_{j=1}^{l} \, U(\hat{\boldsymbol{y}}, \boldsymbol{c}_j) \, p (\boldsymbol{c}_j \,|\, \boldsymbol{x}) \,.
\end{equation}

In the following, we (1) define the Bayes-optimal decision for each utility metric of interest, and
(2) discuss how each loss function can be regarded as a smooth surrogate for the corresponding metric. 

\paragraph{Bitwise loss functions.} 
Define bitwise accuracy as
$\frac{1}{m} \sum_{i=1}^m [y_i = \hat{y}_i]$ with $[\cdot]$ the indicator function.
\citet{dembczynski2012label} have shown that the Bayes-optimal decision for bitwise accuracy is:
\begin{equation*}
    y_i^\star(\boldsymbol{x}) = \argmax_{b \in \{0,1\}} \mathbb{P}(Y_i=b\mid \boldsymbol{x})\\ =  \argmax_{b \in \{0,1\}} \sum_{\boldsymbol{y} \in \mathcal{Y}: y_i = b} p(\boldsymbol{y} \,|\, \boldsymbol{x})\,.
\end{equation*}
Here, $\mathbb{P}(Y_i=b\mid \boldsymbol{x})$ is the marginal probability of the $i$-th bit being $b$. Binary cross entropy \cref{eq:bceloss} and focal loss \cref{eq:flloss} can therefore be viewed as surrogates for bit-wise accuracy \citep{bartlett2006convexity}. 

\paragraph{Vectorwise loss function.}
As performance metrics, the cosine similarity, Tanimoto similarity and F1-measure are closely related.\footnote{In multi-label classification, those three metrics can be computed per instance (instance-wise averaging), per label (macro-averaging), or by pooling all predictions together for a dataset (micro-averaging). Bayes-optimal decisions might differ for the three versions, but only the instance-wise averaged version is relevant for this paper.} All three are fractions with the same numerator $\hat{\boldsymbol{y}}^\top \boldsymbol{y}$, differing only in the denominator. For F1, the optimal prediction can be obtained from $m^2+1$ parameters of the joint distribution but admits no closed form \citep{waegeman2014}. For Tanimoto, the inference problem is believed to be NP-hard \citep{chierchetti2010}.
Cosine similarity, by contrast, has received little attention from the decision-theoretic literature, despite its widespread adoption following advances in self-supervised learning.
Our first theoretical result (proof in \cref{app:Bayes-cosine}) shows that the Bayes-optimal cosine prediction admits a simple closed form:
\begin{restatable}{theorem}{thmcosinebayes}
\label{thm:cosine_bayes}
For $i\in\{1,\dots,m\}$ define 
$w_i \;=\; \sum_{\boldsymbol{y}\in\{0,1\}^m}\frac{y_i\,p(\boldsymbol{y}\,|\,\boldsymbol{x})}{\|\boldsymbol{y}\|}$ and let $w_{(1)}\ge\cdots\ge w_{(m)}$ be the sorted weights. For any fixed size $s$, the
best $s$-sparse Bayes predictor for cosine similarity selects the $s$ indices with largest $w_i$'s
(ties arbitrary). Overall, the optimal support size is 
$s^\star \in \argmax_{s\in\{1,\dots,m\}} \frac{1}{\sqrt{s}}\sum_{t=1}^{s} w_{(t)}\,$.
\end{restatable}

\paragraph{Listwise loss functions.}
For HR@1 the Bayes-optimal decision is the mode of the joint distribution \citep{dembczynski2012label}:
$\boldsymbol{y}^\star_{\HRAtOne}(\boldsymbol{x}) = \argmax_{\boldsymbol{c} \in \mathcal{C}} p(\boldsymbol{c} \,|\, \boldsymbol{x}) \,.$
This is also known as the MAP estimate. In \cref{sec:contrastiveHR} we explain how the contrastive loss in \cref{eq:contrastiveloss} is a computationally efficient alternative to categorical cross-entropy, as a surrogate for HR@1 maximization. Furthermore, we also discuss Bayes-optimal decisions for HR@k. 

\paragraph{Illustrative example.}
\begin{wraptable}{r}{4cm}
\centering
\fbox{
    \begin{tabular}{cc}
		$\boldsymbol{c}_j$     &  $p(\boldsymbol{c}_j\, | \, \boldsymbol{x})$ \\[0.4mm]
		\midrule
		$\boldsymbol{c}_1 = 11000$ & 0.05 \\
		$\boldsymbol{c}_2 = 10100$ & 0.25 \\
		$\boldsymbol{c}_3 = 10010$ & 0.30 \\
		$\boldsymbol{c}_4 = 11001$ & 0.40 \\
	\end{tabular}}
	\label{tab:tanimexample}
\end{wraptable}
These theoretical results already provide justification for why the different methods analyzed in the experiments cannot be optimal for all metrics at the same time.
To illustrate that Bayes-optimal decisions can be different, let us consider a simple example with fingerprint bitvectors of length 5 and $l=4$ non-zero probability candidates ($\mathcal{C}=\{\boldsymbol{c}_1, \dots, \boldsymbol{c}_4\}$), see table on the right.
For this toy example, it is easy to exhaustively enumerate all possible combinations and obtain the Bayes-optimal solution by brute force.
Using the values provided, we find that $\boldsymbol{y}^\star_{\text{HR1}}=11001$, $\boldsymbol{y}^\star_{\text{BIT}}=10000$,
$\boldsymbol{y}^\star_{\text{COS}}=11111$, and $\boldsymbol{y}^\star_{\text{Tan}}=11001$.
This illustrates that, even in such simple cases, the optimal solutions under different metrics can diverge substantially. Our example connects to the notion of regret, i.e.\ the difference between the expected loss of the chosen solution and the minimal achievable Bayes loss~\citep{berger2013statistical}.
In the next section, we generalize this example and derive regret bounds.

\subsection{Regret analysis} \label{sec:regret_analysis}
In this section, we conduct a regret analysis of Bayes-optimal decision rules under different loss functions discussed in \cref{sec:contrastive_as_retrieval} and present novel bounds. 
First, we define the regret of decoding with rule $V$ but evaluating under metric $U$ as
\[
\mathcal{R}_{U}^{\mathrm{vs}\,V}
:= 
\mathbb{E}[U(\bY,\by_U^\star(\bx)) \mid \bx] -
\mathbb{E}[U(\bY,\by_V^\star(\bx)) \mid \bx] \,,
\]
and its worst-case regret over any probability distribution is denoted as $\sup_{p(\cdot\mid \boldsymbol{x})} \mathcal{R}_{U}^{\mathrm{vs}\,V}$.

All the introduced bounds are conditioned on a fixed spectrum $\bx$ and candidate set $\mathcal{C}=\{\bc_1,\ldots,\bc_l\}$ with conditional probabilities
$p_j := p(\bc_j \mid \bx)$ and $\sum_{j=1}^{l} p_j = 1$. Let $p_\star := \max_{j} p_j$ denote the highest conditional probability.
For a similarity metric $U_\Sim(\bc_i, \bc_j) \in [0, 1]$ such that $U_\Sim(\bc_i, \bc_i) = 1$
, we also define the similarity band:
\[
    \sigma_{\min} := \min_{i\neq j} U_\Sim(\bc_i, \bc_j), \quad \sigma_{\max} := \max_{i\neq j} U_\Sim(\bc_i, \bc_j) \,,
\]
A first regret bound concerns the regret regarding any vectorwise similarity metric (e.g. Tanimoto, cosine, F1-measure, etc.) when maximizing HR@1 instead. 
\begin{restatable}[Regret of HR@1 under vectorwise similarity]{theorem}{thmregrethrtovecsim}
\label{thm:regret_hr1_to_vecsim}
When the similarity metric is symmetric, so
$U_\Sim(\bc_i, \bc_j) = U_\Sim(\bc_j, \bc_i)$, the regret of HR@1 decoding evaluated under the similarity metric $U_\Sim$ obeys:
\begin{equation*}
\mathcal{R}_{\Sim}^{\mathrm{vs}\,\HRAtOne} \le \big[ (\sigma_{\max} - \sigma_{\min})(1-p_\star - p_\Sim) 
\IfRestatedTF{}{\\}
- (1-\sigma_{\max})(p_\star - p_\Sim) \big]_+ = \sup_{p(\cdot | \bx)} \mathcal{R}_{\Sim}^{\mathrm{vs}\,\HRAtOne}\,,
\end{equation*}
where $p_\Sim = p(\by^\star_\Sim(\bx) \mid \bx)$ and $[z]_+ := \max\{0,z\}$.
\end{restatable}

Without making additional assumptions about the candidate set (given only $\sigma_{\min}\,,\sigma_{\max}\,, p_\star$ and $p_\Sim$), the above upper bound is tight. 
We present the proof of the theorem and its tightness in \cref{app:regret_hr_to_sim}.
The result shows that regret increases with conditional aleatoric uncertainty $(1-p_\star)$ and
with the width of the candidate-set similarity band $(\sigma_{\max}-\sigma_{\min})$.
A larger margin between $p_\star$ and $p_\Sim$ reduces the regret, but its influence is
attenuated when $\sigma_{\max}$ is close to $1$ (near duplicates in the candidate set),
because the subtractive term involves $(1-\sigma_{\max})$. 
Empirically, in the equal-mass candidate sets of MassSpecGym, we often observe 
a large range of bands and sometimes $\sigma_{\max} \approx 1$,
as shown in \cref{sec:extra_results} -- \cref{fig:sigmas,fig:cosine_sigmas}.
This makes the HR@1 choice particularly vulnerable under Tanimoto/cosine similarity.

A second regret bound concerns the regret regarding HR@1 when maximizing any similarity metric (or minimizing any vectorwise loss) instead. 
\begin{restatable}[Regret of vectorwise similarity under HR@1]{theorem}{thmregretvecsimtohr}
\label{thm:regret_vecsim_to_hr1}
\[
    \mathcal{R}_{\HRAtOne}^{\mathrm{vs}\,\Sim}
    \le \min \left\{p_\star, \frac{\sigma_{\max}-\sigma_{\min}}{1-\sigma_{\max}} (1-p_\star)\right\} \,.
\]
\end{restatable}

We present the proof of the theorem in \cref{app:regret_sim_to_hr}. 
Again, wide similarity bands and near duplicates ($\sigma_{\max}\to 1$) inflate this regret. 
In data regimes with high uncertainty (small $p_\star$),
the second term becomes dominant and can be large.
Based on these results we can also derive a sufficient condition for no regret (agreement of decisions) between two rules, see \cref{app:hr_sim_aggrement} for details.

The last case we consider in regret analysis is the gap between decoding for bitwise accuracy evaluated under HR@1 and vectorwise loss functions.
While the bitwise accuracy can also be viewed as a vectorwise loss, and thus its relation to HR$@1$ regret is covered by \cref{thm:regret_vecsim_to_hr1}, an alternative general worst case result is known~\citep{dembczynski2012label}:
$\sup_{p(\cdot | \bx)} \mathcal{R}_\HRAtOne^{\mathrm{vs\,BIT}}=\tfrac{1}{2}$.
Following this line of constructing regret bounds, we also show that decoding by bitwise accuracy (optimized by BCE/focal) can be arbitrarily misaligned with other vectorwise losses of interest:
\begin{restatable}[Regret bitwise-loss under vectorwise similarity]{theorem}{thmbitwisevectorwise}
\label{thm:bitwisevectorwise}
\[
    \sup_{p(\cdot | \bx)} \mathcal{R}_\Tan^{\mathrm{vs\,\Bit}}\ge\tfrac{1}{2},\qquad
    \sup_{p(\cdot | \bx)} \mathcal{R}_\Cos^{\mathrm{vs\,\Bit}}\ge\tfrac{1}{2}\quad (m\ge 3) \,.
\]
\end{restatable}
\cref{app:worst_case_depending_on_m} contains the proofs of these suprema,
as well as alternative bounds on $\sup \mathcal{R}_\Bit^{\mathrm{vs\,\HRAtOne}}$, 
$\sup \mathcal{R}_\Cos^{\mathrm{vs\,\HRAtOne}}$ and $\sup \mathcal{R}_\Tan^{\mathrm{vs\,\HRAtOne}}$ 
(depending only on the number of fingerprint bits $m$).
In \cref{sec:generalization}, we discuss how all these regret analyses can be expanded to include finite-sample generalization error components.

Together, \Crefrange{thm:regret_hr1_to_vecsim}{thm:bitwisevectorwise} show that loss choice cannot be made independently of the candidate-set structure, and that high conditional aleatoric uncertainty (small $p_\star$) can force any single loss to be suboptimal under at least one metric.
This regime is the rule rather than the exception in metabolomics, where many spectra cannot be unambiguously mapped to a single structure~\citep{bittremieux2022critical}.
The theory thus predicts a fundamental trade-off in practical settings: maximizing fingerprint accuracy comes at the cost of retrieval performance, and vice versa.

\looseness=-1

\section{Experiments} \label{sec:msmolemainres}

In this section, we empirically demonstrate that loss functions trade off retrieval performance with fingerprint accuracy in metabolomics molecular retrieval settings.
All code to reproduce results are present at \url{https://github.com/gdewael/ms-mole}.

\paragraph{Experimental setup.}
For the MassSpecGym dataset~\citep{bushuiev2024massspecgym}, we use the predefined splits ensuring that any two spectra in different splits have a Maximum Common Edge Subgraph (MCES) distance~\citep{kretschmer2025coverage} of at least 10.
The training, validation, and test sets contain 194\,119\,/\,19\,429\,/\,17\,556 spectra over 22\,746\,/\,3\,185\,/\,2\,998 molecules.
MassSpecGym defines two retrieval settings, with candidates of equal mass or equal chemical formula drawn from a database up to $|\mathcal{C}^{(i)}| \leq 256$.
We train a separate model for each setting.
For Enveda-180 (v1.0.0), we generate similar splits by computing the MCES distance between all pairs of its 184\,330 molecules, and performing single-linkage clustering.
As a threshold of 10 groups nearly the entire dataset in one cluster (rendering a train/val/test split impossible), we use an adaptive threshold: a pair of two compounds $a$ and $b$ are clustered together if their MCES distance is equal to or below $\max(3, 0.1\min(\text{nHA(a), \text{nHA}(b)}))$, where $\text{nHA}(\cdot)$ yields the number of heavy atoms of a compound.
This yields training, validation, and test sets that are similarly leakage-proof~(\cref{sec:extra_results} -- \Cref{fig:envedaleakage}) and containing 936\,835\,/\,115\,091\,/\,106\,367 spectra over 149\,054\,/\,17\,969\,/\,17\,307 molecules.
For Enveda-180, we define equal-mass candidate sets in the same way as MassSpecGym: random selection from PubChem equal-mass compounds ($\pm$10 ppm tolerance) up until a max of 256.
Unless otherwise noted, experiments represent molecules by 4096-dimensional Morgan fingerprints~\citep{rogers2010extended} (radius 2).

As explained in \cref{sec:eval_fp}, fingerprint predictions are produced according to $\hat{\boldsymbol{y}}^{(i)} = f(\boldsymbol{x}^{(i)}) = \sigma(\boldsymbol{W}_\text{out} E(\boldsymbol{x}^{(i)}))\in[0,1]^{4096}$.
In this study, we test spectrum embedders $E$ both as (1) MLPs operating on binned input spectra $\boldsymbol{x}^{(i)}$ bounded by \textit{m/z} values in the interval $[0, 1005]$, with a binning width of $0.1$, and (2) transformer encoders operating on peak tokens using sinusoidal encodings for intensities and \textit{m/z} values.
Architectural details are given in \cref{sec:suppmethodsmsmole}.

To optimize all models, we use one of the loss functions explained in \cref{sec:loss_fp}.
For the contrastive loss \cref{eq:contrastiveloss}, a function $g$ produces a logit from spectrum and molecule representations.
We vary $g$ along two axes: (i) fingerprint space vs. a learned projection (via $\boldsymbol{W}_\text{sp}\in\mathbb{R}^{h \times d}$ and $\boldsymbol{W}_\text{c}\in\mathbb{R}^{h \times m}$ projecting spectra and candidates to an $h$-dimensional space), and (ii) cosine similarity vs. a learned MLP. The four combinations are:
\begin{align}
    g_{\text{Fp-Cos}}\big(\boldsymbol{x}^{(i)}, \boldsymbol{c}_j^{(i)}\big) = & \cos\big(f(\boldsymbol{x}^{(i)}), \boldsymbol{c}_j^{(i)}\big)\,, \label{eq:fp-cos} \\
    g_{\text{Emb-Cos}}\big(\boldsymbol{x}^{(i)}, \boldsymbol{c}_j^{(i)}\big) = & \cos\big(\boldsymbol{W}_\text{sp}E(\boldsymbol{x}^{(i)}), \boldsymbol{W}_\text{c}\boldsymbol{c}_j^{(i)}\big)\,, \label{eq:emb-cos}\\
    g_{\text{Fp-MLP}}\big(\boldsymbol{x}^{(i)}, \boldsymbol{c}_j^{(i)}\big) =&\  \text{MLP}\big(f(\boldsymbol{x}^{(i)}), \boldsymbol{c}_j^{(i)}\big)\,, \label{eq:fp-mlp}\\
    g_{\text{Emb-MLP}}\big(\boldsymbol{x}^{(i)}, \boldsymbol{c}_j^{(i)}\big) =&\  \text{MLP}\big(\boldsymbol{W}_\text{sp}E(\boldsymbol{x}^{(i)}), \boldsymbol{W}_\text{c}\boldsymbol{c}_j^{(i)}\big) \,. \label{eq:emb-mlp}
\end{align}
In existing works, the formulation using cosine similarity on a lower-dimensional space ("Emb-Cos") is most popular~\citep{goldman2023annotating, kalia2025jestr}.

In total, eight loss functions were tested: (1) BCE, (2) FL, (3) Cosine Loss, (4) IoU Loss, (5) Contrastive FP-Cos, (6) Contrastive Emb-Cos, (7) Contrastive Fp-MLP, and (8) Contrastive Emb-MLP.
For every loss function, the learning rate is tuned separately on MassSpecGym.
Test results are then shown using the models with the optimal learning rate, validation-wise.
For the focal and all contrastive losses, the additional hyperparameters of $\gamma$ and $\tau$, respectively, are additionally tuned.
Model training and tuning details are given in \cref{sec:suppmethodsmsmole}.

Model performance is evaluated using (1) the average Tanimoto similarity of the prediction to the fingerprint of the true molecule, and (2) the average hit rate @ $k$ on the two candidate settings.
For computing validation and test Tanimoto similarities, all predictions are binarized using $\hat{\boldsymbol{y}}^{(i)}>0.5$.
For computing retrieval scores, continuous $[0,1]$-bounded predictions are used.

\paragraph{Experimental results.}

\cref{tab:msmole_maintabv2} summarizes the results.
Generally speaking, Transformer spectrum embedders outperform MLPs.\footnotemark{}
Enveda-180 results are generally higher than those for MassSpecGym.
This could be due to (1) Enveda-180 covering a closer compound domain than the broad MassSpecGym, (2) the uniformly acquired nature and quality of Enveda-180 spectra as opposed to the mix of libraries in MassSpecGym, and (3) the overall larger size of Enveda-180.
The best retrieval scores are obtained with contrastive losses and the versions using the cosine similarity score (either on the fingerprint, or on the embedding level) outperform the learned MLP similarity variants.
The best average Tanimoto similarities are obtained using the IoU Loss.
Unsurprisingly, each loss wins on the metric it optimizes for.
However, we empirically find that the two objectives of retrieval and fingerprint accuracy (in terms of Tanimoto similarity) seem to oppose one another: the best fingerprint predictors are the worst retrievers, and vice versa.
This trade-off forms a Pareto front (\cref{fig:inverse_corr}).

\footnotetext{
We additionally perform experiments with the pretrained DreaMS transformer backbone, finding no consistent improvement over transformers trained from scratch~(\cref{sec:extra_results} -- \cref{tab:dreams}).
}

\begin{table*}[t!]
    \centering
    \setlength{\tabcolsep}{4pt}
    \caption{
        Molecular retrieval and fingerprint accuracy performance on MassSpecGym~\citep{bushuiev2024massspecgym} and Enveda-180~\citep{krettler2026enveda}.
        The best and second-best performing models for each metric are indicated with boldface and underlined, respectively.
        For the models trained in this study, numbers report averages and standard deviations over 5 model runs.
        The performances of a random model and the MIST model~\citep{goldman2023annotating} are taken from \citet{bushuiev2024massspecgym}.
        For contrastive models, Tanimoto scores are only reported for the "FP-Cos" variant, as this is the only one that can be interpreted as predicting a true fingerprint.
    }
    \resizebox{1.0\linewidth}{!}{
        \begin{tabular}{lccccccccccc}
            \toprule
            &  \multicolumn{7}{c}{MassSpecGym} & \multicolumn{4}{c}{Enveda-180} \\ \cmidrule(lr){2-8} \cmidrule(lr){9-12}
            &  \multicolumn{3}{c}{Equal mass candidates}  & \multicolumn{3}{c}{Equal formula candidates}  & & \multicolumn{3}{c}{Equal mass candidates} \\
            \cmidrule(lr){2-4} \cmidrule(lr){5-7} \cmidrule(lr){9-11}
            &  \multicolumn{1}{c}{HR@1} & \multicolumn{1}{c}{HR@5} & \multicolumn{1}{c}{HR@20} & \multicolumn{1}{c}{HR@1} & \multicolumn{1}{c}{HR@5} & \multicolumn{1}{c}{HR@20} & \multicolumn{1}{c}{Tanimoto} & \multicolumn{1}{c}{HR@1} & \multicolumn{1}{c}{HR@5} & \multicolumn{1}{c}{HR@20} & \multicolumn{1}{c}{Tanimoto} \\ \midrule \multicolumn{8}{l}{\textit{MLP (This study)}} \\
            BCE & \res{10.12}{0.36} & \res{19.44}{0.16} & \res{34.96}{0.73}  & \res{10.08}{0.22}  & \res{22.60}{0.35}   & \res{40.79}{0.41}           & \res{16.75}{0.11} & \res{35.43}{0.09} & \res{55.13}{0.06} & \res{66.11}{0.07} & \res{23.01}{0.05} \\
            FL  & \res{11.54}{0.14} & \res{20.84}{0.27} & \res{36.72}{0.44} & \res{11.15}{0.14} & \res{23.82}{0.23} & \res{42.76}{0.53} & \res{16.58}{0.16} & \res{40.01}{0.05} & \res{65.71}{0.06} & \res{83.12}{0.09} & \res{21.91}{0.04} \\
            Cosine Loss  & \res{\z9.48}{0.19} & \res{18.19}{0.17} & \res{32.90}{0.38} & \res{\z9.12}{0.41} & \res{21.33}{0.30} & \res{39.38}{0.31} & \res{17.22}{0.09} & \res{23.53}{0.04} & \res{38.57}{0.05} & \res{55.64}{0.13} & \res{25.26}{0.03} \\
            IoU Loss    & \res{\z5.58}{0.32} & \res{13.01}{0.39} & \res{26.51}{0.50} & \res{\z6.48}{0.32} & \res{17.62}{0.51} & \res{35.05}{0.22} & \res{19.42}{0.13} & \res{20.37}{0.02} & \res{35.99}{0.05} & \res{50.72}{0.06} & \res{28.34}{0.03} \\
            Contr. Fp-Cos   & \res{12.30}{0.67}	& \res{26.38}{0.95}	& \res{42.65}{2.18}     & \res{11.37}{0.31} &	\res{25.28}{0.36} &	\res{45.04}{0.76}  & \res{\z7.86}{0.99}$^1$ & \res{44.19}{0.02} & \res{64.56}{0.06} & \res{79.32}{0.03} & \res{\z2.39}{0.01}$^1$ \\
            Contr. Emb-Cos & \res{12.29}{0.69} &	\res{26.09}{1.43} &	\res{45.05}{0.71}     & \res{13.62}{1.12}	& \res{26.33}{1.06} &	\res{46.50}{0.94}     & \multicolumn{1}{c}{N/A} & \res{45.89}{0.04} & \res{71.11}{0.09} & \res{87.82}{0.08} & \multicolumn{1}{c}{N/A} \\
            Contr. Fp-MLP & \res{\z9.72}{0.20} &	\res{22.99}{1.80} &	\res{42.34}{1.08}      & \res{12.28}{0.79} & \res{24.34}{0.37} & \res{43.49}{0.80}         & \multicolumn{1}{c}{N/A} & \res{41.18}{0.07} & \res{66.48}{0.09} & \res{85.61}{0.07} & \multicolumn{1}{c}{N/A} \\
            Contr. Emb-MLP & \res{11.08}{0.40} &	\res{25.66}{0.82} &	\res{43.91}{2.23}      & \res{12.46}{0.63} &	 \res{25.15}{0.89} &	\res{44.72}{1.39}        & \multicolumn{1}{c}{N/A} & \res{44.69}{0.05} & \res{66.89}{0.06} & \res{83.21}{0.10} & \multicolumn{1}{c}{N/A} \\ \midrule \multicolumn{8}{l}{\textit{Transformer (This study)}} \\
            BCE & \res{15.47}{0.28} &	\res{30.08}{0.57} &	\res{50.02}{0.24} & \res{12.98}{0.52} &	\res{28.83}{0.80} &	\res{49.89}{0.64} & \res{21.07}{0.25} & \res{53.01}{0.33} & \res{68.20}{0.15} & \res{79.28}{0.23} & \res{32.49}{0.06} \\
            FL & \second{17.03}{0.37} &	\res{31.63}{0.62} &	\res{51.87}{0.50} & \res{14.22}{0.24} &	\res{29.92}{0.38} &	\res{50.27}{0.56} & \res{21.27}{0.13} & \res{57.30}{0.09} & \res{72.42}{0.11} & \res{82.49}{0.17} & \res{31.93}{0.08} \\
            Cosine Loss & \res{14.31}{0.52} &	\res{31.35}{0.30} &	\res{51.14}{0.16} & \res{13.52}{0.35} &	\res{29.21}{0.34} &	\res{50.37}{0.47} & \second{21.92}{0.17} &  \res{50.88}{0.03} & \res{66.12}{0.09} & \res{77.23}{0.13}  & \second{33.26}{0.07} \\
            IoU Loss & \res{\z9.43}{0.50} &	\res{20.57}{0.72} &	\res{38.80}{0.43} & \res{09.41}{0.45} &	\res{22.32}{0.64} &	\res{41.46}{0.42} & \best{22.13}{0.15} & \res{43.14}{0.01} & \res{57.02}{0.02} & \res{69.88}{0.07} & \best{34.38}{0.04} \\
            Contr. Fp-Cos & \best{18.50}{0.38} &	\best{36.91}{0.30} &	\res{58.62}{0.27} & \second{14.79}{0.34} &	\second{31.39}{0.74} &	\second{52.78}{0.64} & \res{\z8.13}{0.67}$^1$  & \best{63.11}{0.05} & \second{81.04}{0.09} & \res{89.97}{0.11} & \res{\z4.64}{0.05}$^1$ \\
            Contr. Emb-Cos & \res{16.13}{0.69} &	\res{31.46}{0.70} &	\res{53.73}{1.41} & \best{15.28}{1.11} &	\res{30.07}{1.05} &	\res{50.35}{0.64} & \multicolumn{1}{c}{N/A} & \second{63.05}{0.02} & \best{82.54}{0.10} & \best{92.98}{0.06} & \multicolumn{1}{c}{N/A} \\
            Contr. Fp-MLP & \res{12.44}{0.43} &	\res{29.14}{0.14} &	\res{53.34}{0.57} & \res{14.04}{0.59} &	\res{27.14}{0.70} &	\res{47.76}{1.92} & \multicolumn{1}{c}{N/A} & \res{59.43}{0.30} & \res{78.93}{0.53} & \res{87.25}{0.60} & \multicolumn{1}{c}{N/A} \\
            Contr. Emb-MLP & \res{16.96}{0.72} &	\res{34.29}{0.96} &	\res{55.05}{1.47} & \res{14.37}{1.25} &	\res{29.64}{0.62} &	\res{50.42}{0.46} & \multicolumn{1}{c}{N/A} & \res{61.90}{0.21} & \res{81.01}{0.14} & \second{90.66}{0.26} & \multicolumn{1}{c}{N/A}
            \\\midrule \multicolumn{8}{l}{\textit{MassSpecGym \citep{bushuiev2024massspecgym} baselines}} \\
            Random         & $\z0.37$          & $\z2.01$          & $\z8.22$           & $\z3.06$           & $11.35$          & $27.74$           &  \multicolumn{1}{c}{-} & \multicolumn{1}{c}{-} & \multicolumn{1}{c}{-} & \multicolumn{1}{c}{-} & \multicolumn{1}{c}{-}\\
            MIST$^2$           & $14.64$ $^2$          & $34.87$ $^2$         & $\underline{59.15}$ $^2$          & $\z9.57$           & $22.11$          & $41.12$ & \multicolumn{1}{c}{-} & \multicolumn{1}{c}{-} & \multicolumn{1}{c}{-} & \multicolumn{1}{c}{-} & \multicolumn{1}{c}{-}\\
            JESTR~\citep{kalia2025jestr} & $15.13$ & $\underline{36.75}$ & $\boldsymbol{60.32}$ & $11.85$ & $\boldsymbol{32.95}$ & $\boldsymbol{61.46}$ & 
            \multicolumn{1}{c}{-} & \multicolumn{1}{c}{-} & \multicolumn{1}{c}{-} & \multicolumn{1}{c}{-} & \multicolumn{1}{c}{-}\\ \bottomrule
        \end{tabular}
    } \begin{minipage}{1.0\linewidth}
    	\scriptsize
    	\raggedright
    	$^1$: Obtained using a different temperature value ($\tau=4$) than used for retrieval scores ($\tau=1/256$), as optimal values for the two objectives drastically differ. For more details, see Appendix~\cref{tab:tau_tune}. \\
        $^2$: MIST scores were obtained by running the model with ground truth molecular formulas. In the mass-based retrieval setting, this can be interpreted as either a form of data leakage, or as an "oracle" setting in which the formula is known. For more information, see \href{https://github.com/pluskal-lab/MassSpecGym/issues/52}{MassSpecGym GitHub Issue \#52}.
    	\end{minipage}\label{tab:msmole_maintabv2}
\end{table*}

The above results were obtained after tuning the focal parameter $\gamma$ and contrastive temperature $\tau$ (\cref{sec:extra_results} -- \cref{tab:gamma_tune,tab:tau_tune}).
All results used the optimal hyperparameter values chosen in terms of HR@20.
An exception was made for the contrastive Fp-Cos model, as an inverse correlation was found between retrieval scores and Tanimoto similarities as a function of the temperature hyperparameter.
For this reason, two versions of models using this loss were trained separately using optimal temperatures for both metrics.
In addition, validation performances as a function of training time for a selection of loss functions are shown in \cref{sec:extra_results} -- \cref{fig:val_over_time}.

In practice, retrieval results are better when using cosine similarity to score candidates against a prediction, regardless of the loss function used during training (see \cref{sec:extra_results} -- \cref{tab:retrieval_func}).
Hence, \cref{tab:msmole_maintabv2} uses cosine similarity scoring throughout.
An exception to this rule is the contrastive loss function, which adopts the learned similarity $g$.
In addition, we present experimental results using (1) different choices for negative candidate sets in \cref{sec:supp_hardnegs}, and (2) using weighted combinations of contrastive and IoU loss terms in \cref{sec:extra_results} -- \cref{fig:inverse_corr2}.
To illustrate that the similarity bands themselves are a driver of the trade-off, in \cref{tab:partition}, we partition the MassSpecGym test set into halves based on the $\sigma_{\max} - \sigma_{\min}$ and $\sigma_{\max}$ of each spectrum's candidate set and analyze how the trade-off differs in each half.

\paragraph{How theory informs fingerprint selection.}

Our regret bounds (\cref{thm:regret_hr1_to_vecsim,thm:regret_vecsim_to_hr1}) depend on two factors: the conditional probabilities $p(\boldsymbol{y}~|~\boldsymbol{x})$, which reflect target predictability, and the similarity bands $\sigma_{\min}$, $\sigma_{\max}$ of the candidate set.
The latter are computable from data without training a model, making them useful for fingerprint design: smaller $\sigma_{\max}$ and $\sigma_{\max} - \sigma_{\min}$ should reduce regret and narrow the gap between fingerprint-similarity optimization and retrieval optimization.
This matters because many published fingerprint predictors optimize for similarity~\citep{baygi2024idsl_mint}, despite retrieval being the actual downstream goal.
To test whether different fingerprints exhibit different similarity bands, we compare five 4096-bit alternatives to Morgan ($r=2$): Morgan with $r\in\{4,6,8\}$, RDKit~\citep{landrum2013rdkit}, and MAP4~\citep{capecchi2020one}.
Empirical $\sigma$ distributions on MassSpecGym candidate sets (\cref{tab:sigma_distros}) show that MAP4 and larger-radius Morgan fingerprints carry the most discriminative similarity bands, while RDKit frequently admits high-similarity negatives.

Retraining IoU and contrastive (Emb-Cos) transformer models on each target (\cref{tab:fp_selection}), the HR@1 gap between the two losses broadly tracks the $\sigma$ statistics, supporting the theory's first factor.
The Morgan-radius series, however, exposes the role of the second factor.
Although $\sigma$ statistics improve monotonically with radius, retrieval performance worsens just as monotonically: contrastive HR@1 drops from 18.50 at $r=2$ to 9.65 at $r=8$.
The Tanimoto column clarifies why.
Average Tanimoto similarity to ground truth drops sharply (22.13 at $r=2$ to 12.88 at $r=8$), indicating that larger radii are harder prediction targets.
Harder targets plausibly correspond to more uniform $p(\boldsymbol{y}~|~\boldsymbol{x})$, which inflates the $(1-p^*)$ terms in \cref{thm:regret_hr1_to_vecsim,thm:regret_vecsim_to_hr1}.
Favorable similarity bands alone are not sufficient: the regret bound also requires predictable targets.

MAP4 is the only fingerprint we tested for which both factors are favorable: more discriminative $\sigma$ bands than Morgan ($r=2$), at comparable Tanimoto similarity.
Among fingerprints with retrieval performance comparable to Morgan ($r=2$), MAP4 has the narrowest IoU-vs-contrastive HR@1 gap.
Our theory and experiments thus identify MAP4 as a promising fingerprint target for molecular retrieval, complementary or superior to the \textit{de facto} Morgan ($r=2$) standard.
More broadly, this analysis turns fingerprint choice into a tunable component of the retrieval pipeline: any fingerprint pairing narrow $\sigma$ bands with high target predictability should reduce regret.

\begin{table}[t!]
    \centering
    \setlength{\tabcolsep}{4pt}
    \caption{
        MassSpecGym (equal-mass) results using different fingerprint targets.
        Numbers report averages and standard deviations over 5 model runs.
        The best and second-best performing models for each metric are indicated with boldface and underlined, respectively.
    }
    \resizebox{1.0\linewidth}{!}{
        \begin{tabular}{lccccccc}
            \toprule
            & \multicolumn{4}{c}{Scores obtained with IoU loss}                                      & \multicolumn{3}{c}{Scores obtained with Contr. Fp-Cos loss}                                               \\ \cmidrule(lr){2-5} \cmidrule(lr){6-8}
            Fingerprint type & Tanimoto$^1$ & HR@1 & HR@5 & HR@20 & HR@1 & HR@5 & HR@20 \\ \midrule
            Morgan $r=2$  & $\underline{22.13 \pm 0.15}$ & $09.43 \pm 0.50$ &	$20.57 \pm 0.72$ &	$\underline{38.80 \pm 0.43}$	& $\boldsymbol{18.50 \pm 0.38}$ &	$\boldsymbol{36.91 \pm 0.30}$ &	$\boldsymbol{58.62 \pm 0.27}$   \\
            Morgan $r=4$   & $15.64 \pm 0.16$ &	$09.27 \pm 0.30$ &	$20.22 \pm 0.19$ &	$38.06 \pm 0.54$  & $11.90 \pm 1.28$ &	$27.55 \pm 0.64$ &	$47.28 \pm 1.50$  \\
            Morgan $r=6$   & $13.52 \pm 0.04$ &	$08.34 \pm 0.56$ &	$18.60 \pm 0.70$ &	$36.45 \pm 0.75$ 	&$10.33 \pm 0.29$ &	$24.70 \pm 0.65$ &	$44.77 \pm 1.66$  \\
            Morgan $r=8$   & $12.88 \pm 0.07$ &	$08.17 \pm 0.49$ &	$19.02 \pm 0.76$ &	$36.94 \pm 0.74$ & $09.65 \pm 0.91$ &	$23.06 \pm 1.35$ &	$43.73 \pm 2.00$ \\
            RDKit   & $\boldsymbol{34.89 \pm 0.14}$ &	$\underline{09.51 \pm 0.20}$ &	$\underline{20.71 \pm 0.30}$ &	$\boldsymbol{39.21 \pm 0.54}$ &$13.38 \pm 0.59$ &	$29.21 \pm 0.96$ &	$52.57 \pm 1.15$ \\
            MAP4   & $21.09 \pm 0.17$ &	$\boldsymbol{11.79 \pm 0.25}$ &	$\boldsymbol{22.56 \pm 0.36}$ &	$38.25 \pm 0.39$	& $\underline{18.14 \pm 0.62}$ &	$\underline{35.91 \pm 1.32}$ &	$\underline{57.52 \pm 0.76}$ \\
            \bottomrule
        \end{tabular}
    } \begin{minipage}{1.0\linewidth}
    	\scriptsize
    	\raggedright
    	$^1$: Comparing Tanimoto scores between fingerprint types should be performed with care, as they concern different targets. If anything, these scores reflect how ``predictable'' a given fingerprint type is.
    	\end{minipage}
        \label{tab:fp_selection}
\end{table}

\section{Discussion}

In this study, we examine a consistent Pareto trade‑off between (1) fingerprint similarity and (2) molecular retrieval: improving one typically degrades the other (\cref{fig:inverse_corr}).
Using Bayes-risk analysis, we have theoretically explained why these outcomes are expected and defined regret bounds on how different the optimal decisions by each model may be.

Most current research uses modeling strategies that optimize for predicting correct fingerprints~\citep{duhrkop2015searching, goldman2023annotating}.
All the while, these models have molecular retrieval as their implicit goal.
Our experimental results straightforwardly show that, if the goal is to perform molecular retrieval, the loss function should be chosen with this task in mind (i.e., using a contrastive loss).
Moreover, our theoretical results confirm that the regret between the decisions for either (1) optimal similarity or (2) optimal retrieval can be large.
It is trivial to show that, when the conditional class distribution is degenerate (i.e., $p(\boldsymbol{y}~|~\boldsymbol{x})=1$ for the true candidate), there is no regret.
As such, our empirical results can only be explained in the context of assuming uncertainty in the output class distribution.
It is known that molecular retrieval is a generally hard problem because of both (1) the insufficiency of signal in the input mass spectra, and (2) the vast complexity of chemical space~\citep{da2015illuminating}.
These two issues connect to the notions of aleatoric and epistemic uncertainty, respectively~\citep{hullermeier2021aleatoric}.
For this reason, we consider uncertainty estimation to be a valuable future research direction in this field.

Minimizing regret is desirable, as obtaining models excelling at both tasks is practically useful.
\cref{thm:regret_hr1_to_vecsim,thm:regret_vecsim_to_hr1}
show that for some choices of similarity function and/or fingerprints type, the regret may actually be small.
This finding connects to an emerging area of research broadly investigating how the choice of molecular representation impacts model performance~\citep{huber2025count}.
Here, we presented results using alternative fingerprints, revealing that MAP4 performs comparably to Morgan ($r=2$) fingerprints, but with lower regrets.
We anticipate that these theoretical insights will inform subsequent research on fingerprint selection.

Finally, we note that our regret bounds hold for any generic similarity function operating on bitvector candidate sets.
As such, the trade-off we characterize may therefore arise in other domains where binary representations are used for similarity-based retrieval, such as item retrieval using learned hashes.
Empirical verification in such settings remains for future work.

\begin{ack}
W.W. received funding from the Flemish Government under the Flanders AI Research programme. W.K. was supported from the HORIZON-MSCA-2021-DN-01 project LEMUR no. 101073307, funded by the European Union. G.D.W. was supported by Research Foundation-Flanders (FWO) [PhD Fellowship fundamental research grant 1153024N].

\end{ack}

\bibliographystyle{unsrtnat}
\bibliography{refs}

@article{stravs2024decade,
	title={A Decade of Computational Mass Spectrometry from Reference Spectra to Deep Learning},
	author={Stravs, Michael A},
	journal={Chimia},
	volume={78},
	number={7-8},
	pages={525--530},
	year={2024}
}

@article{waegeman2014,
author = {Waegeman, Willem and Dembczy\'{n}ki, Krzysztof and Jachnik, Arkadiusz and Cheng, Weiwei and H\"{u}llermeier, Eyke},
title = {On the bayes-optimality of F-measure maximizers},
year = {2014},
issue_date = {January 2014},
publisher = {JMLR.org},
volume = {15},
number = {1},
issn = {1532-4435},
journal = {J. Mach. Learn. Res.},
month = jan,
pages = {3333–3388},
numpages = {56}
}

@article{Mortier2023,
author = {Mortier, Thomas and Wydmuch, Marek and Dembczy\'{n}ski, Krzysztof and H\"{u}llermeier, Eyke and Waegeman, Willem},
title = {Efficient set-valued prediction in multi-class classification},
year = {2021},
issue_date = {Jul 2021},
publisher = {Kluwer Academic Publishers},
address = {USA},
volume = {35},
number = {4},
issn = {1384-5810},
url = {https://doi.org/10.1007/s10618-021-00751-x},
doi = {10.1007/s10618-021-00751-x},
journal = {Data Min. Knowl. Discov.},
month = jul,
pages = {1435–1469},
numpages = {35}
}

@InProceedings{Ye2012,
  Title                    = {Optimizing {F}-measures: a tale of two approaches},
  Author                   = {N. Ye and K. Chai and W. Lee and H. Chieu},
  Booktitle                = {Proceedings of the International Conference on Machine Learning},
  Year                     = {2012}
}

@inproceedings{Hazan2010,
 author = {Hazan, Tamir and Keshet, Joseph and McAllester, David},
 booktitle = {Advances in Neural Information Processing Systems},
 editor = {J. Lafferty and C. Williams and J. Shawe-Taylor and R. Zemel and A. Culotta},
 pages = {},
 publisher = {Curran Associates, Inc.},
 title = {Direct Loss Minimization for Structured Prediction},
 url = {https://proceedings.neurips.cc/paper_files/paper/2010/file/ca8155f4d27f205953f9d3d7974bdd70-Paper.pdf},
 volume = {23},
 year = {2010}
}

@inproceedings{
bohde2025diffms,
title={Diff{MS}: Diffusion Generation of Molecules Conditioned on Mass Spectra},
author={Montgomery Bohde and Mrunali Manjrekar and Runzhong Wang and Shuiwang Ji and Connor W. Coley},
booktitle={Forty-second International Conference on Machine Learning},
year={2025},
url={https://openreview.net/forum?id=EvILcv2v8L}
}

@article{Berman2017TheLL,
  title={The Lovasz-Softmax Loss: A Tractable Surrogate for the Optimization of the Intersection-Over-Union Measure in Neural Networks},
  author={Maxim Berman and A. Triki and Matthew B. Blaschko},
  journal={2018 IEEE/CVF Conference on Computer Vision and Pattern Recognition},
  year={2017},
  pages={4413-4421},
  url={https://api.semanticscholar.org/CorpusID:4716955}
}

@article{Patel2021RecallkSL,
  title={Recall@k Surrogate Loss with Large Batches and Similarity Mixup},
  author={Yash J. Patel and Giorgos Tolias and Jiri Matas},
  journal={2022 IEEE/CVF Conference on Computer Vision and Pattern Recognition (CVPR)},
  year={2021},
  pages={7492-7501},
  url={https://api.semanticscholar.org/CorpusID:237291532}
}

@inproceedings{Menon2019,
 author = {Menon, Aditya K and Rawat, Ankit Singh and Reddi, Sashank and Kumar, Sanjiv},
 booktitle = {Advances in Neural Information Processing Systems},
 editor = {H. Wallach and H. Larochelle and A. Beygelzimer and F. d\textquotesingle Alch\'{e}-Buc and E. Fox and R. Garnett},
 pages = {},
 publisher = {Curran Associates, Inc.},
 title = {Multilabel reductions: what is my loss optimising?},
 url = {https://proceedings.neurips.cc/paper_files/paper/2019/file/da647c549dde572c2c5edc4f5bef039c-Paper.pdf},
 volume = {32},
 year = {2019}
}

@InProceedings{dembczynski13,
  title = 	 {Optimizing the F-Measure in Multi-Label Classification: Plug-in Rule Approach versus Structured Loss Minimization},
  author = 	 {Dembczynski, Krzysztof and Jachnik, Arkadiusz and Kotlowski, Wojciech and Waegeman, Willem and Huellermeier, Eyke},
  booktitle = 	 {Proceedings of the 30th International Conference on Machine Learning},
  pages = 	 {1130--1138},
  year = 	 {2013},
  editor = 	 {Dasgupta, Sanjoy and McAllester, David},
  volume = 	 {28},
  number =       {3},
  series = 	 {Proceedings of Machine Learning Research},
  address = 	 {Atlanta, Georgia, USA},
  month = 	 {17--19 Jun},
  publisher =    {PMLR},
  url = 	 {https://proceedings.mlr.press/v28/dembczynski13.html}
}

@inproceedings{chierchetti2010,
author = {Chierichetti, Flavio and Kumar, Ravi and Pandey, Sandeep and Vassilvitskii, Sergei},
title = {Finding the Jaccard median},
year = {2010},
isbn = {9780898716986},
publisher = {Society for Industrial and Applied Mathematics},
address = {USA},
booktitle = {Proceedings of the Twenty-First Annual ACM-SIAM Symposium on Discrete Algorithms},
pages = {293–311},
numpages = {19},
location = {Austin, Texas},
series = {SODA '10}
}

@article{de2022good,
	title={Good practices and recommendations for using and benchmarking computational metabolomics metabolite annotation tools},
	author={de Jonge, Niek F and Mildau, Kevin and Meijer, David and Louwen, Joris JR and Bueschl, Christoph and Huber, Florian and van der Hooft, Justin JJ},
	journal={Metabolomics},
	volume={18},
	number={12},
	pages={103},
	year={2022},
	publisher={Springer}
}

@article{bushuiev2024massspecgym,
	title={MassSpecGym: A benchmark for the discovery and identification of molecules},
	author={Bushuiev, Roman and Bushuiev, Anton and de Jonge, Niek and Young, Adamo and Kretschmer, Fleming and Samusevich, Raman and Heirman, Janne and Wang, Fei and Zhang, Luke and D{\"u}hrkop, Kai and others},
	journal={Advances in Neural Information Processing Systems},
	volume={37},
	pages={110010--110027},
	year={2024}
}

@article{bajusz2015tanimoto,
	title={Why is Tanimoto index an appropriate choice for fingerprint-based similarity calculations?},
	author={Bajusz, D{\'a}vid and R{\'a}cz, Anita and H{\'e}berger, K{\'a}roly},
	journal={Journal of cheminformatics},
	volume={7},
	pages={1--13},
	year={2015},
	publisher={Springer}
}

@article{butler2023ms2mol,
	title={MS2Mol: A transformer model for illuminating dark chemical space from mass spectra},
	author={Butler, Thomas and Frandsen, Abraham and Lightheart, Rose and Bargh, Brian and Taylor, James and Bollerman, TJ and Kerby, Thomas and West, Kiana and Voronov, Gennady and Moon, Kevin and others},
	year={2023}
}

@article{da2015illuminating,
	title={Illuminating the dark matter in metabolomics},
	author={da Silva, Ricardo R and Dorrestein, Pieter C and Quinn, Robert A},
	journal={Proceedings of the National Academy of Sciences},
	volume={112},
	number={41},
	pages={12549--12550},
	year={2015},
	publisher={National Academy of Sciences}
}

@article{huber2021ms2deepscore,
	title={MS2DeepScore: a novel deep learning similarity measure to compare tandem mass spectra},
	author={Huber, Florian and van der Burg, Sven and van der Hooft, Justin JJ and Ridder, Lars},
	journal={Journal of cheminformatics},
	volume={13},
	number={1},
	pages={84},
	year={2021},
	publisher={Springer}
}

@article{nam2017maximizing,
	title={Maximizing subset accuracy with recurrent neural networks in multi-label classification},
	author={Nam, Jinseok and Loza Menc{\'\i}a, Eneldo and Kim, Hyunwoo J and F{\"u}rnkranz, Johannes},
	journal={Advances in neural information processing systems},
	volume={30},
	year={2017}
}

@article{kretschmer2025coverage,
	title={Coverage bias in small molecule machine learning},
	author={Kretschmer, Fleming and Seipp, Jan and Ludwig, Marcus and Klau, Gunnar W and B{\"o}cker, Sebastian},
	journal={Nature Communications},
	volume={16},
	number={1},
	pages={554},
	year={2025},
	publisher={Nature Publishing Group UK London}
}

@article{hullermeier2021aleatoric,
	title={Aleatoric and epistemic uncertainty in machine learning: An introduction to concepts and methods},
	author={H{\"u}llermeier, Eyke and Waegeman, Willem},
	journal={Machine Learning},
	volume={110},
	pages={457--506},
	year={2021},
	publisher={Springer}
}

@article{rogers2010extended,
	title={Extended-connectivity fingerprints},
	author={Rogers, David and Hahn, Mathew},
	journal={Journal of chemical information and modeling},
	volume={50},
	number={5},
	pages={742--754},
	year={2010},
	publisher={ACS Publications}
}

@inproceedings{cao2007learning,
	title={Learning to rank: from pairwise approach to listwise approach},
	author={Cao, Zhe and Qin, Tao and Liu, Tie-Yan and Tsai, Ming-Feng and Li, Hang},
	booktitle={Proceedings of the 24th international conference on Machine learning},
	pages={129--136},
	year={2007}
}

@article{mikolov2013distributed,
	title={Distributed representations of words and phrases and their compositionality},
	author={Mikolov, Tomas and Sutskever, Ilya and Chen, Kai and Corrado, Greg S and Dean, Jeff},
	journal={Advances in neural information processing systems},
	volume={26},
	year={2013}
}

@inproceedings{chen2020simple,
	title={A simple framework for contrastive learning of visual representations},
	author={Chen, Ting and Kornblith, Simon and Norouzi, Mohammad and Hinton, Geoffrey},
	booktitle={International conference on machine learning},
	pages={1597--1607},
	year={2020},
	organization={PmLR}
}

@article{wang2023jaccard,
	title={Jaccard metric losses: optimizing the Jaccard index with soft labels},
	author={Wang, Zifu and Ning, Xuefei and Blaschko, Matthew},
	journal={Advances in Neural Information Processing Systems},
	volume={36},
	pages={75259--75285},
	year={2023}
}

@article{baygi2024idsl_mint,
  title={IDSL\_MINT: A deep learning framework to predict molecular fingerprints from mass spectra},
  author={Baygi, Sadjad Fakouri and Barupal, Dinesh Kumar},
  journal={Journal of Cheminformatics},
  volume={16},
  number={1},
  pages={8},
  year={2024},
  publisher={Springer}
}

@inproceedings{nowozin2014optimal,
	title={Optimal decisions from probabilistic models: the intersection-over-union case},
	author={Nowozin, Sebastian},
	booktitle={Proceedings of the IEEE conference on computer vision and pattern recognition},
	pages={548--555},
	year={2014}
}

@article{landrum2013rdkit,
  title={Rdkit documentation},
  author={Landrum, Greg},
  journal={Release},
  volume={1},
  number={1-79},
  pages={4},
  year={2013}
}

@article{capecchi2020one,
  title={One molecular fingerprint to rule them all: drugs, biomolecules, and the metabolome},
  author={Capecchi, Alice and Probst, Daniel and Reymond, Jean-Louis},
  journal={Journal of cheminformatics},
  volume={12},
  number={1},
  pages={43},
  year={2020},
  publisher={Springer}
}

@article{huber2025count,
  title={Count your bits: more subtle similarity measures using larger radius count vectors},
  author={Huber, Florian and Pollmann, Julian},
  journal={bioRxiv},
  pages={2025--06},
  year={2025},
  publisher={Cold Spring Harbor Laboratory}
}

@article{duhrkop2015searching,
	title={Searching molecular structure databases with tandem mass spectra using CSI: FingerID},
	author={D{\"u}hrkop, Kai and Shen, Huibin and Meusel, Marvin and Rousu, Juho and B{\"o}cker, Sebastian},
	journal={Proceedings of the National Academy of Sciences},
	volume={112},
	number={41},
	pages={12580--12585},
	year={2015},
	publisher={National Academy of Sciences}
}

@article{cereto2015molecular,
	title={Molecular fingerprint similarity search in virtual screening},
	author={Cereto-Massagu{\'e}, Adri{\`a} and Ojeda, Mar{\'\i}a Jos{\'e} and Valls, Cristina and Mulero, Miquel and Garcia-Vallv{\'e}, Santiago and Pujadas, Gerard},
	journal={Methods},
	volume={71},
	pages={58--63},
	year={2015},
	publisher={Elsevier}
}

@article{dreams,
	title={Emergence of molecular structures from repository-scale self-supervised learning on tandem mass spectra},
	DOI={10.26434/chemrxiv-2023-kss3r-v2},
	journal={ChemRxiv},
	author={Bushuiev, Roman and Bushuiev, Anton and Samusevich, Raman and Brungs, Corinna and Sivic, Josef and Pluskal, Tomáš},
	year={2024}
}

@article{goldman2023annotating,
	title={Annotating metabolite mass spectra with domain-inspired chemical formula transformers},
	author={Goldman, Samuel and Wohlwend, Jeremy and Stra{\v{z}}ar, Martin and Haroush, Guy and Xavier, Ramnik J and Coley, Connor W},
	journal={Nature Machine Intelligence},
	volume={5},
	number={9},
	pages={965--979},
	year={2023},
	publisher={Nature Publishing Group UK London}
}

@article{dembczynski2012label,
	title={On label dependence and loss minimization in multi-label classification},
	author={Dembczy{\'n}ski, Krzysztof and Waegeman, Willem and Cheng, Weiwei and H{\"u}llermeier, Eyke},
	journal={Machine Learning},
	volume={88},
	pages={5--45},
	year={2012},
	publisher={Springer}
}

@inproceedings{lin2017focal,
	title={Focal loss for dense object detection},
	author={Lin, Tsung-Yi and Goyal, Priya and Girshick, Ross and He, Kaiming and Doll{\'a}r, Piotr},
	booktitle={Proceedings of the IEEE international conference on computer vision},
	pages={2980--2988},
	year={2017}
}

@article{krettler2026enveda,
  title={Enveda-180: generation of a large open multimodal MS/MS and ion mobility spectral library for drug-like small molecules},
  author={Krettler, Christoph A and Kind, Tobias and DeBloois, Erik and Killian, Marie and Prince, John and Taylor, James and Gadiya, Yojana and Simpson, Pelle and Haynes, Sarah and Domingo-Fern{\'a}ndez, Daniel and others},
  year={2026}
}

@article{zhou2026learning,
  title={Learning from All Views: A Multiview Contrastive Framework for Metabolite Annotation},
  author={Zhou Chen, Yan and Hassoun, Soha},
  journal={Analytical Chemistry},
  volume={98},
  number={9},
  pages={6598--6606},
  year={2026},
  publisher={ACS Publications}
}

@article{bohde2026frigid,
  title={FRIGID: Scaling Diffusion-Based Molecular Generation from Mass Spectra at Training and Inference Time},
  author={Bohde, Montgomery and Liu, Hongxuan and Manjrekar, Mrunali and Lederbauer, Magdalena and Ji, Shuiwang and Wang, Runzhong and Coley, Connor W},
  journal={arXiv preprint arXiv:2604.16648},
  year={2026}
}

@article{han2025ms,
  title={MS-BART: Unified Modeling of Mass Spectra and Molecules for Structure Elucidation},
  author={Han, Yang and Wang, Pengyu and Yu, Kai and Chen, Xin and Chen, Lu},
  journal={arXiv preprint arXiv:2510.20615},
  year={2025}
}

@article{bittremieux2022critical,
  title={The critical role that spectral libraries play in capturing the metabolomics community knowledge},
  author={Bittremieux, Wout and Wang, Mingxun and Dorrestein, Pieter C},
  journal={Metabolomics},
  volume={18},
  number={12},
  pages={94},
  year={2022},
  publisher={Springer}
}

@article{stravs2022msnovelist,
	title={MSNovelist: de novo structure generation from mass spectra},
	author={Stravs, Michael A and D{\"u}hrkop, Kai and B{\"o}cker, Sebastian and Zamboni, Nicola},
	journal={Nature Methods},
	volume={19},
	number={7},
	pages={865--870},
	year={2022},
	publisher={Nature Publishing Group US New York}
}

@article{bartlett2006convexity,
  title={Convexity, classification, and risk bounds},
  author={Bartlett, Peter L and Jordan, Michael I and McAuliffe, Jon D},
  journal={Journal of the American Statistical Association},
  volume={101},
  number={473},
  pages={138--156},
  year={2006},
  publisher={Taylor \& Francis}
}

@book{berger2013statistical,
  title={Statistical decision theory and Bayesian analysis},
  author={Berger, James O},
  year={2013},
  publisher={Springer Science \& Business Media}
}

@book{devroye2013probabilistic,
  title={A probabilistic theory of pattern recognition},
  author={Devroye, Luc and Gy{\"o}rfi, L{\'a}szl{\'o} and Lugosi, G{\'a}bor},
  volume={31},
  year={2013},
  publisher={Springer Science \& Business Media}
}

@inproceedings{yang2020consistency,
  title={On the consistency of top-k surrogate losses},
  author={Yang, Forest and Koyejo, Sanmi},
  booktitle={International Conference on Machine Learning},
  pages={10727--10735},
  year={2020},
  organization={PMLR}
}

@article{chen2024cmssp,
  title={Cmssp: A contrastive mass spectra-structure pretraining model for metabolite identification},
  author={Chen, Lu and Xia, Bing and Wang, Yu and Huang, Xia and Gu, Yucheng and Wu, Wenlin and Zhou, Yan},
  journal={Analytical Chemistry},
  volume={96},
  number={42},
  pages={16871--16881},
  year={2024},
  publisher={ACS Publications}
}

@article{kalia2025jestr,
  title={JESTR: J oint E mbedding S pace T echnique for R anking Candidate Molecules for the Annotation of Untargeted Metabolomics Data},
  author={Kalia, Apurva and Zhou Chen, Yan and Krishnan, Dilip and Hassoun, Soha},
  journal={Bioinformatics},
  pages={btaf354},
  year={2025},
  publisher={Oxford University Press}
}

@article{hendrycks2016gaussian,
	title={Gaussian error linear units (gelus)},
	author={Hendrycks, Dan and Gimpel, Kevin},
	journal={arXiv preprint arXiv:1606.08415},
	year={2016}
}

@article{ba2016layer,
	title={Layer normalization},
	author={Ba, Jimmy Lei and Kiros, Jamie Ryan and Hinton, Geoffrey E},
	journal={arXiv preprint arXiv:1607.06450},
	year={2016}
}

@article{srivastava2014dropout,
	title={Dropout: a simple way to prevent neural networks from overfitting},
	author={Srivastava, Nitish and Hinton, Geoffrey and Krizhevsky, Alex and Sutskever, Ilya and Salakhutdinov, Ruslan},
	journal={The journal of machine learning research},
	volume={15},
	number={1},
	pages={1929--1958},
	year={2014},
	publisher={JMLR. org}
}

\appendix

\section{Related Work} \label{app:related_work}

\paragraph{Fingerprint prediction and molecular retrieval from tandem mass-spectrometry}
Early computational approaches for small-molecule identification from MS/MS data framed the task as molecular property prediction, in which a model predicts a binary structural fingerprint used for database retrieval. A foundational line of work is CSI:FingerID~\citep{duhrkop2015searching}, which trains a large collection of support vector machines to predict individual fingerprint bits using molecular fragmentation tree-based inputs. CSI:FingerID is optimized explicitly for bitwise accuracy, and its fingerprints are used with Tanimoto similarity for retrieval. Although influential, CSI:FingerID predates large supervised datasets such as MassSpecGym, and its learning setup (bit-independent classifiers, handcrafted features, and no end-to-end neural representation) differs substantially from the deep learning architectures evaluated in this study.

More recent work has moved toward neural models that learn joint representations of spectra and structures.
MIST~\citep{goldman2023annotating} trains a chemical-formula-aware transformer to predict fingerprints in two stages.
First, it optimizes for fingerprint accuracy using the cosine loss, after which a fine-tuning stage optimizes for retrieval with a contrastive loss.
Final retrieval is performed via embedding-space cosine similarity rather than via predicted binary fingerprints.
Because MIST is trained with listwise contrastive objectives and uses learned projections of candidate fingerprints, its final optimization target aligns directly with HR@1 retrieval, not with fingerprint similarity. Our contrastive Emb-Cos variants mirror the contrastive fine-tuning setup of MIST.
Interestingly, in their study, the initial MIST fingerprint predictor optimized with cosine loss outperforms CSI:FingerID in fingerprint prediction accuracy but underperforms it in retrieval, mirroring the trade-off we examine here.
Other recent methods (e.g., CMSSP~\citet{chen2024cmssp}, JESTR~\citet{kalia2025jestr}, MVP~\citep{zhou2026learning}) also employ contrastive or metric-learning formulations, again emphasizing retrieval-optimized training objectives over fingerprint accuracy.

Beyond molecular retrieval, a growing body of work utilizes fingerprint predictors to generate conditioning embeddings for \textit{de novo} structure elucidation~\citep{bohde2025diffms, han2025ms, bohde2026frigid}.
As it stands, it remains unclear how to best optimize the fingerprint predictor in such two-stage settings.

\section{Proof of Bayes-optimal prediction for cosine similarity}
\label{app:Bayes-cosine}

We assume $p(\boldsymbol{y}\,|\,\boldsymbol{x})=0$ whenever $\|\boldsymbol{y}\|=0$, so that
$U_\Cos(\hat{\boldsymbol y},\boldsymbol y)$ and the quantities below are well-defined.

\thmcosinebayes*


\begin{proof}
For any nonzero $\hat{\boldsymbol y}\in\{0,1\}^m$ with support $S(\hat{\boldsymbol y})=\{i:\hat y_i=1\}$ and size
$s(\hat{\boldsymbol y})=|S(\hat{\boldsymbol y})|$, the expected cosine similarity satisfies
\begin{align*}
\mathbb{E}\!\left[U_\Cos(\hat{\boldsymbol y},\boldsymbol y)\right]
&=\sum_{\boldsymbol y}\frac{\sum_{i=1}^m \hat y_i y_i}{\|\hat{\boldsymbol y}\|\,\|\boldsymbol y\|}\,p(\boldsymbol y\,|\,\boldsymbol x)
=\frac{1}{\|\hat{\boldsymbol y}\|}\sum_{i=1}^m \hat y_i \sum_{\boldsymbol y}\frac{y_i\,p(\boldsymbol y\,|\,\boldsymbol x)}{\|\boldsymbol y\|} \\
&=\frac{1}{\|\hat{\boldsymbol y}\|}\sum_{i=1}^m \hat y_i w_i
=\frac{1}{\sqrt{s(\hat{\boldsymbol y})}}\sum_{i\in S(\hat{\boldsymbol y})} w_i.
\end{align*}
by definition of cosine similarity and linearity of expectation. As a result, maximizing the expected cosine similarity over binary predictions is equivalent to maximizing
\[
F(S)\;=\;\frac{1}{\sqrt{|S|}}\sum_{i\in S} w_i
\qquad\text{over }S\subseteq\{1,\dots,m\},\ S\neq\varnothing,
\]
where $S$ is the support of the prediction. Let $S^\star$ be the support of a maximizer and
assume, for contradiction, that there exist $i\in S^\star$ and $j\notin S^\star$ with $w_i < w_j$.
Define $S' = S^\star \cup \{j\}\setminus\{i\}$, which has the same cardinality as $S^\star$.
Then
\[
F(S')-F(S^\star) \;=\; \frac{1}{\sqrt{|S^\star|}}\big(w_j-w_i\big) \;>\; 0,
\]
contradicting optimality. Therefore $w_i\ge w_j$ whenever $i\in S^\star$ and $j\notin S^\star$.
\end{proof}

\section{The contrastive loss is a differentiable surrogate for HR@1} \label{sec:contrastiveHR}

To establish a formal link between the contrastive losses and retrieval, let us consider that a smooth surrogate for optimizing the hit rate @ $1$ could be to maximize the probability of the true element being ranked first (i.e., the top-1 probability)~\citep{yang2020consistency}.
According to ListNet~\citep{cao2007learning}, top-1 probability $P_{\mathcal{S}}(a)$ of an item $a$ given a set of scores $\mathcal{S}$ is given by:
\begin{equation*}
 	P_{\mathcal{S}}(a) = \frac{\phi(s_a)}{\sum_{s_j\in\mathcal{S}} \phi(s_j)},
\end{equation*}
where $\phi(\cdot)$ can be any increasing and strictly positive function.
When $\phi(\cdot) = \exp(\cdot)$, the top-1 probability is estimated via the softmax operation.

Given a single object as ground truth "positive" molecule (i.e., $\boldsymbol{y}$), the top-one probability can, hence, be maximized by minimizing the following cross-entropy loss $L$:
\begin{equation*}
 	L(\mathcal{S}) = - \log(P_{\mathcal{S}}(\boldsymbol{y})) = -\log\frac{\exp(s_{\boldsymbol{y}})}{\sum_{s_j\in\mathcal{S}} \exp(s_j)}.
\end{equation*}

Let us now recall that the contrastive loss \cref{eq:contrastiveloss} employs a function $g(\boldsymbol{x}, \boldsymbol{c}_j)$ to obtain a matching score between any input spectrum and candidate structure $\boldsymbol{c}_j\in\mathcal{C}$.
Substituting the scores $s$ in the equation above for the function $g$ essentially derives the contrastive loss.
As a result, optimizing contrastive losses is equivalent to the retrieval objective of top-1 probability maximization, and, consequently, HR@1 maximization.

The HR@1 and contrastive loss can also be regarded as subset 0/1 accuracy maximization in multi-class classification.
In the multi-class classification setting, a surrogate loss for subset 0/1 maximization would be to use traditional categorical cross-entropy to estimate the mode \citep{dembczynski2012label}.
Enumerating all possible molecules as classes is computationally infeasible, as argued in \cref{sec:loss_fp}.
The similarity function inside the contrastive loss allows for efficient computation in an exponential search space with millions of molecules \citep{cao2007learning}.
Even if one could define all possible molecules as classes, one would have an extremely sparse training signal.
Many valid molecules might never appear in your training data, making it hard for the model to learn to predict them.
Standard categorical cross-entropy with softmax primarily focuses on maximizing the probability of the correct class (i.e., the entire correct structured output) and minimizing the probability of all incorrect classes.
It doesn't inherently understand the internal relationships or similarities between different molecules.
However, such a compromise between computational feasibility, expressive power and overfitting mitigation characterizes various types of probabilistic models for structured outputs, such as conditional random fields, probabilistic hierarchical classifiers \citep{Mortier2023} and autoregressive neural network decoders \citep{nam2017maximizing}. 

\paragraph{On HR@k} Note that the formal link between the contrastive loss in this study and the retrieval objective is limited to the HR@1 case.
Here, we provide some preliminary discussion on the differences between HR@1 and HR@k in terms of (1) Bayes-optimal decisions, and (2) their optimization.
To analyze Bayes-optimal decisions for HR@k, the notation in \cref{eq:bayes_opt} needs to be extended, so that the algorithm returns $k$ molecules instead of one. Then, the result for HR@1 can be immediately extended to HR@k \citep{yang2020consistency}. Order the candidates in $\mathcal{C}$ by decreasing conditional probability: $p(\boldsymbol{c}_1 \,|\, \boldsymbol{x}) \geq p(\boldsymbol{c}_2 \,|\, \boldsymbol{x}) \geq \cdots \geq p(\boldsymbol{c}_l \,|\, \boldsymbol{x}) \,.$
Then, the Bayes-optimal decision for HR@k is given by the $k$ highest molecules in this list. 
Furthermore, \citet{yang2020consistency} showed that top‑$k$ calibration is necessary and sufficient for consistency; they prove several hinge‑style top‑$k$ surrogates proposed in earlier computer‑vision work are not top‑$k$ calibrated, while giving consistent alternatives. Cross‑entropy (in rich function classes) is top‑$k$ consistent, but may lose consistency under linear restrictions. These results formalize when utility maximization (via a surrogate) recovers the Bayes decoder for top‑$k$. As a side note, remark that HR@$k$ has also been theoretically analyzed in multi-label retrieval~\citep{Menon2019,Patel2021RecallkSL}, but these results are less relevant for mass spectrometry data.

\section{Proofs for regret analysis theorems} \label{app:regret_proofs}


For mathematical rigor, we first extend the definitions of the utilities from the main text to handle special cases.  Let $\mathcal{Y}=\{0,1\}^m$ with $m\ge 2$. For $\boldsymbol{y},\hat{\boldsymbol{y}}\in\mathcal{Y}$ define:
\begin{align*}
U_\Bit(\boldsymbol{y},\hat{\boldsymbol{y}})
&= \frac{1}{m}\sum_{i=1}^m \indicator\left[y_i=\hat{y}_i\right], \\
U_\HRAtOne(\boldsymbol{y},\hat{\boldsymbol{y}})
&= \indicator\left[\boldsymbol{y}=\hat{\boldsymbol{y}}\right], \\
U_\Tan(\boldsymbol{y},\hat{\boldsymbol{y}})
&= \begin{cases}
\dfrac{\|\boldsymbol{y}\wedge \hat{\boldsymbol{y}}\|_0}{\|\boldsymbol{y}\vee \hat{\boldsymbol{y}}\|_0}, & \|\boldsymbol{y}\vee \hat{\boldsymbol{y}}\|_0 > 0, \\
1, & \boldsymbol{y}=\hat{\boldsymbol{y}}=\boldsymbol{0}, 
\end{cases} \\
U_\Cos(\boldsymbol{y},\hat{\boldsymbol{y}})
&= \begin{cases}
\dfrac{\boldsymbol{y}^\top \hat{\boldsymbol{y}}}{\|\boldsymbol{y}\|\,\|\hat{\boldsymbol{y}}\|}, & \|\boldsymbol{y}\|\cdot\|\hat{\boldsymbol{y}}\| > 0, \\
1, & \boldsymbol{y}=\hat{\boldsymbol{y}}=\boldsymbol{0}, \\
0, & \text{otherwise}.
\end{cases}
\end{align*}
Here $\|\cdot\|_0$ denotes the $\ell_0$ “norm” and $\|\cdot\|$ the Euclidean norm.

Given a conditional distribution $p(\boldsymbol{y}\mid \boldsymbol{x})$, the Bayes predictor for utility $U$ is defined as
\[
\boldsymbol{y}_U^\star(\boldsymbol{x}) \in \arg\max_{\hat{\boldsymbol{y}}\in\mathcal{Y}} \;
\mathbb{E}_{\boldsymbol{Y}\sim p(\cdot\mid \boldsymbol{x})}\big[ U(\boldsymbol{Y},\hat{\boldsymbol{y}}) \big].
\]

\subsection{Proof of HR@1 decoding under the vectorwise similarity}
\label{app:regret_hr_to_sim}

Let denote the expected $\Sim$-utility of predicting $\bc_i$ as:
\[
\label{eq:expected_u}
u_i \;:=\; \E_{\bY\sim p(\cdot\mid \bx)}\!\big[U_\Sim(\bc_i,\bY)\big]
\;=\; \sum_{j=1}^{l} U_\Sim(\bc_i, \bc_j ) p_j .
\]
Let $i_\star := \argmax_i p_i$ be the index of fingerprint from the candidate set $\mathcal{C}$ 
that has the highest posterior $p_\star = p_{i_\star} = \max_i p_i$.
Then let $i_\Sim := \argmax_i u_i$ be the index of fingerprint that is $\Sim$-optimal (vectorwise) decision,
so \(\by^\star_\Sim(\bx)=\bc_{i_\Sim}\) and \(p_\Sim:=p(\by^\star_\Sim(\bx)\mid\bx)=p_{i_\Sim}\).

\begin{lemma}[Row‑wise bounds]
\label{lem:row-bounds}
For every $i \in \{1,2,\ldots,l\}$,
\begin{equation}
p_i + \sigma_{\min}(1-p_i) \le u_i \le p_i + \sigma_{\max}(1-p_i)
= \sigma_{\max} + (1-\sigma_{\max}) p_i \,.
\label{eq:row_wise_bound}
\end{equation}
\end{lemma}

\begin{proof}
For any index $i\in\{1,\ldots,l\}$, by definition
\begin{equation}
u_i = \sum_{j=1}^{l} U_\Sim(\bc_i, \bc_j )p_j
= U_\Sim(\bc_i, \bc_i)\,p_i + \sum_{j\neq i} U_\Sim(\bc_i, \bc_j )p_j
= p_i + \sum_{j\neq i} U_\Sim(\bc_i, \bc_j )p_j \,.
\label{eq:ui-split}
\end{equation}
Let $t := \sum_{j\neq i} p_j = 1-p_i$. 
If $t=0$ (equivalently $p_i=1$) then  $u_i = p_i = 1$, and both inequalities in \cref{eq:row_wise_bound} hold with equality.
In the case $t>0$ define normalized nonnegative weights
\begin{equation*}
w_j := \frac{p_j}{t}, \qquad j\neq i,
\end{equation*}
so that $\sum_{j\neq i} w_j = 1$ and $w_j \ge 0$ for all $j\neq i$.
Using these weights, Eq.~\cref{eq:ui-split} becomes
\begin{equation}
\label{eq:ui-weights}
u_i = p_i + t \sum_{j\neq i} U_\Sim(\bc_i, \bc_j )\, w_j .
\end{equation}
By the similarity band, for every $j \neq i$ we have
$\sigma_{\min} \le U_\Sim(\bc_i, \bc_j ) \le \sigma_{\max}$. Therefore, for the convex combination
$\sum_{j\neq i} U_\Sim(\bc_i, \bc_j )\,w_j$ we have the elementary bounds
\begin{equation*}
\sigma_{\min} \;\le\; \sum_{j\neq i} U_\Sim(\bc_i, \bc_j )\,w_j \;\le\; \sigma_{\max}.
\end{equation*}
Multiplying by $t=1-p_i$ and adding $p_i$ to all sides yields
\begin{align*}
p_i + \sigma_{\min}(1-p_i)
&\le p_i + t \sum_{j\neq i} U_\Sim(\bc_i, \bc_j )\, w_j \le p_i + \sigma_{\max}(1-p_i) = \sigma_{\max} + (1-\sigma_{\max})\,p_i,
\end{align*}
\end{proof}

\thmregrethrtovecsim*

\begin{proof}
The regret can be simply written as
\[
\mathcal{R}_{\Sim}^{\mathrm{vs}\,\HRAtOne}
= \E[U_\Sim(\bY,\by^\star_\Sim(\bx))\mid\bx]-\E[U_\Sim(\bY,\by^\star_{\HRAtOne}(\bx))\mid\bx]
= u_{i_\Sim} - u_{i_\star} \,.
\]
If \(i_\Sim=i_\star\) there is no regret. Otherwise, by symmetry let
$\sigma := U_\Sim(\bc_{i_\Sim},\bc_{i_\star}) = U_\Sim(\bc_{i_\star},\bc_{i_\Sim})$.
Using $U_\Sim(\bc_i,\bc_i)=1$, we expand
\begin{align*}
\MoveEqLeft
u_{i_\Sim} - u_{i_\star} 
= \sum_{j=1}^{l} U_\Sim(\bc_{i_\Sim},\bc_j) p_j - \sum_{j=1}^{l} U_\Sim(\bc_{i_\star},\bc_j) p_j \\
&= \left(p_\Sim + \sigma p_\star + \!\!\! \sum_{j \notin \{i_\Sim,i_\star\}} \!\!\! U_\Sim(\bc_{i_\Sim},\bc_j) p_j \right) - \left(p_\star + \sigma p_\Sim + \!\!\! \sum_{j \notin \{i_\Sim,i_\star\}} \!\!\! U_\Sim(\bc_{i_\star},\bc_j) p_j \right).
\end{align*}
Bounding the off-diagonal sums using the similarity band yields
\begin{align*}
u_{i_\Sim} - u_{i_\star} 
&\le \Big(p_\Sim + \sigma\, p_\star + \sigma_{\max} 
\sum_{j \notin \{i_\Sim,i_\star\}} p_j \Big)
~-~ \Big(p_\star + \sigma\, p_\Sim
+ \sigma_{\min} \sum_{j \notin \{i_\Sim,i_\star\}} p_j \Big) \\
&= \sigma(p_\star - p_\Sim) \;+\; (p_\Sim - p_\star) \;+\; (\sigma_{\max} - \sigma_{\min})
\sum_{j \notin \{i_\Sim,i_\star\}} p_j .
\end{align*}
Because $\sum_{j \notin \{i_\Sim,i_\star\}} p_j = 1 - p_\Sim - p_\star$, we get
\begin{align*}
u_{i_\Sim} - u_{i_\star} 
&\le \sigma(p_\star - p_\Sim) \;+\; (p_\Sim - p_\star) \;+\; (\sigma_{\max} - \sigma_{\min})
(1 - p_\Sim - p_\star) \\
&= \underbrace{\big[\sigma(p_\star - p_\Sim) - (p_\star - p_\Sim)\big]}_{(\sigma-1)(p_\star - p_\Sim)}
\;+\; (\sigma_{\max} - \sigma_{\min})
(1 - p_\Sim - p_\star) \\
&= (\sigma - 1)(p_\star - p_\Sim) \;+\; (\sigma_{\max} - \sigma_{\min})
(1 - p_\Sim - p_\star) .
\end{align*}
Since $p_\star \ge p_\Sim$, the right-hand side is nondecreasing in $\sigma$, hence it is upper-bounded by replacing $\sigma$ with its maximal allowable value $\sigma_{\max}$:
\begin{align*}
u_{i_\Sim} - u_{i_\star} 
&\le (\sigma_{\max} - 1)(p_\star - p_\Sim) \;+\; (\sigma_{\max} - \sigma_{\min})
(1 - p_\Sim - p_\star) \\
&= - (1-\sigma_{\max})(p_\star - p_\Sim) \;+\; (\sigma_{\max} - \sigma_{\min}) (1 - p_\Sim - p_\star) \\
&= (\sigma_{\max} - \sigma_{\min})(1 - p_\star - p_\Sim) \;-\; (1-\sigma_{\max})(p_\star - p_\Sim). 
\end{align*}
Taking the outer $[\cdot]_+$ to enforce non-negativity of regret yields inequality in \cref{thm:regret_hr1_to_vecsim}.

To prove tightness of achieved upper-bound let us consider following candidate set.
Let $i_\star$ and $i_\Sim$ be two distinct indices. 
Distribute the remaining mass $1-p_\star-p_\Sim$ across the other $l-2$ indices so that each $p_k$ is small (e.g., all equal), which ensures $u_{i_\Sim}\ge u_k$. Define a candidate set with similarities such that:
\[
\begin{cases}
U_\Sim(\bc_i, \bc_i) = 1 & \text{for all } i,\\[1mm]
U_\Sim(\bc_{i_\Sim},\bc_t)=\sigma_{\max} & \text{for all } t\neq i_\Sim,\\[1mm]
U_\Sim(\bc_{i_\star},\bc_t)=\sigma_{\min} & \text{for all } t \notin \{i_\star,i_\Sim\},\\[1mm]
U_\Sim(\bc_k,\bc_t)=\sigma_{\min} & \text{for all distinct } k,t \notin \{i_\star,i_\Sim\}.
\end{cases}
\]
Note that symmetry forces $U_\Sim(\bc_{i_\star},\bc_{i_\Sim})=U_\Sim(\bc_{i_\Sim},\bc_{i_\star})=\sigma_{\max}$.
Now with these similarities, we show that the regret is equal to the upper-bound.
\begin{align*}
u_{i_\Sim} - u_{i_\star} 
&= \sum_{j=1}^{l} U_\Sim(\bc_{i_\Sim},\bc_j) p_j - \sum_{j=1}^{l} U_\Sim(\bc_{i_\star},\bc_j) p_j \\
&= \left(p_\Sim + U_\Sim(\bc_{i_\Sim},\bc_{i_\star}) p_\star + \!\!\! \sum_{j \notin \{i_\Sim,i_\star\}} \!\!\! U_\Sim(\bc_{i_\Sim},\bc_j) p_j \right) \\
&\qquad- \left(p_\star + U_\Sim(\bc_{i_\star},\bc_{i_\Sim}) p_\Sim + \!\!\! \sum_{j \notin \{i_\Sim,i_\star\}} \!\!\! U_\Sim(\bc_{i_\star},\bc_j) p_j \right) \\
&= \left(p_\Sim + \sigma_{\max} p_\star + \sigma_{\max}\!\!\sum_{t\notin\{i_\Sim,i_\star\}} p_t \right)
 - \left(p_\star + \sigma_{\max} p_\Sim + \sigma_{\min}\!\!\sum_{t\notin\{i_\Sim,i_\star\}} p_t \right) \\
&= \sigma_{\max}(p_\star - p_\Sim) + (p_\Sim - p_\star)
  + (\sigma_{\max} - \sigma_{\min}) \sum_{t\notin\{i_\Sim,i_\star\}} p_t \\
&= \sigma_{\max}(p_\star - p_\Sim) + (p_\Sim - p_\star)
  + (\sigma_{\max} - \sigma_{\min}) (1 - p_\star - p_\Sim) \\
&= (\sigma_{\max} - \sigma_{\min}) (1 - p_\star - p_\Sim) - (1-\sigma_{\max})(p_\star - p_\Sim),
\end{align*}
which is exactly the upper-bound (without the $[\cdot]_+$).

\end{proof}

\subsection{Proof of vectorwise similarity decoding under HR@1}
\label{app:regret_sim_to_hr}

\thmregretvecsimtohr*


\begin{proof}

Recall that $U_{\HRAtOne}(\bY,\hat{\by})=\indicator[\bY=\hat{\by}]$. Hence, for any candidate $\bc_i$,
\[
\E[U_{\HRAtOne}(\bY,\bc_i)\mid\bx] = \mathbb{P}(\bY=\bc_i\mid\bx)=p_i.
\]
Therefore the HR@1 Bayes predictor is $\by^\star_{\HRAtOne}(\bx)=\bc_{i_\star}$ with
expected utility $p_\star=\max_i p_i$. The $\Sim$-optimal predictor is $\by^\star_\Sim(\bx)=\bc_{i_\Sim}$,
whose expected HR@1 utility equals $p_\Sim:=p_{i_\Sim}$. Consequently,
\[
\mathcal{R}_{\HRAtOne}^{\mathrm{vs}\,\Sim}
=\E[U_{\HRAtOne}(\bY,\by^\star_{\HRAtOne}(\bx))\mid\bx]-\E[U_{\HRAtOne}(\bY,\by^\star_{\Sim}(\bx))\mid\bx]
= p_\star - p_\Sim.
\]
The trivial bound $p_\star-p_\Sim\le p_\star$ gives the first term in the $\min\{\cdot,\cdot\}$.
It remains to prove the second term. If $\sigma_{\max}=1$, the fraction
$\frac{\sigma_{\max}-\sigma_{\min}}{1-\sigma_{\max}}$ is undefined and the bound reduces to
$\mathcal{R}_{\HRAtOne}^{\mathrm{vs}\,\Sim}\le p_\star$, which we already established.

By optimality, $u_{i_\Sim}\ge u_{i_\star}$. We apply the row-wise \emph{upper} bound Eq. \cref{eq:row_wise_bound} to $u_{i_\Sim}$ and the
\emph{lower} bound to $u_{i_\star}$:
\begin{align*}
u_{i_\Sim} 
&\le \sigma_{\max} + (1-\sigma_{\max})\,p_{i_\Sim}, \\
u_{i_\star} 
&\ge p_\star + \sigma_{\min}(1-p_\star)
= p_\star + \sigma_{\min} - \sigma_{\min} p_\star.
\end{align*}
Since $u_{i_\Sim}\ge u_{i_\star}$, we have
\begin{align*}
\sigma_{\max} + (1-\sigma_{\max})\,p_{i_\Sim} 
&\;\ge\; p_\star + \sigma_{\min}(1-p_\star) \\
&\;=\; p_\star + \sigma_{\min} - \sigma_{\min} p_\star.
\end{align*}
Subtracting $\sigma_{\max}$ from both sides gives
\begin{equation*}
(1-\sigma_{\max})\,p_{i_\Sim} \;\ge\; p_\star + \sigma_{\min} - \sigma_{\min} p_\star - \sigma_{\max}.
\end{equation*}
Rewriting the right-hand side by grouping the $p_\star$ terms,
\begin{align*}
p_\star + \sigma_{\min} - \sigma_{\min} p_\star - \sigma_{\max}
&= p_\star(1-\sigma_{\min}) + \sigma_{\min} - \sigma_{\max} \\
&= p_\star(1-\sigma_{\min}) - (\sigma_{\max}-\sigma_{\min}).
\end{align*}
Hence
\begin{equation*}
(1-\sigma_{\max})\,p_{i_\Sim} \;\ge\; p_\star(1-\sigma_{\min}) - (\sigma_{\max}-\sigma_{\min}).
\end{equation*}
Since $1-\sigma_{\max}>0$ (off-diagonal similarities are strictly $<1$), we can divide both sides by $1-\sigma_{\max}$ to get:
\begin{equation*}
p_{i_\Sim} \ \ge\ \frac{p_\star(1-\sigma_{\min}) - (\sigma_{\max}-\sigma_{\min})}{\,1-\sigma_{\max}\,}.
\end{equation*}
Therefore,
\begin{align*}
p_\star - p_{i_\Sim}
&\le p_\star - \frac{p_\star(1 - \sigma_{\min}) - (\sigma_{\max} - \sigma_{\min})}{1 - \sigma_{\max}} \\
&= \frac{p_\star(1 - \sigma_{\max})}{1 - \sigma_{\max}}
   \;-\; \frac{p_\star(1 - \sigma_{\min}) - (\sigma_{\max} - \sigma_{\min})}{1 - \sigma_{\max}} \\
&= \frac{p_\star(1 - \sigma_{\max}) - p_\star(1 - \sigma_{\min}) + (\sigma_{\max} - \sigma_{\min})}{1 - \sigma_{\max}} \\
&= \frac{-p_\star(\sigma_{\max} - \sigma_{\min}) + (\sigma_{\max} - \sigma_{\min})}{1 - \sigma_{\max}} \\
&= \frac{(\sigma_{\max} - \sigma_{\min})(1 - p_\star)}{1 - \sigma_{\max}}.
\end{align*}
\end{proof}

\subsection{A sufficient condition for no regret agreement of decisions}
\label{app:hr_sim_aggrement}

\begin{theorem}[Agreement of top-1 and vector-wise choices]
\label{thm:agreement}
Let $\Delta := p_{i_\star} - \max_{j\neq i_\star} p_j \in [0,p_{i_\star}]$ be the margin between the highest and the second highest probability.
Then $i_\Sim = i_\star$; i.e., HR@1 and the $\Sim$-optimal decisions coincide if:
\begin{equation}
\label{eq:agreement-condition}
\Delta \ \ge\ \frac{\sigma_{\max}-\sigma_{\min}}{\,1-\sigma_{\max}\,}\,(1-p_\star) \,.
\end{equation}
\end{theorem}

\begin{proof}
Assume the agreement condition \cref{eq:agreement-condition} holds.
Because under our assumption on the similarity band we have $1-\sigma_{\max} \ge 0$, multiplying both sides of \cref{eq:agreement-condition} by $1-\sigma_{\max}$ gives:
\begin{align*}
   (1 - \sigma_{\max})\Delta 
   &\ge (\sigma_{\max} - \sigma_{\min})(1-p_\star) \\
   0 &\ge (\sigma_{\max} - \sigma_{\min})(1 - p_\star) - (1-\sigma_{\max}) \Delta \,.
\end{align*}

Let $p_{(2)}:=\max_{j\neq i_\star} p_j$ be the
second highest conditional probability. 
By definition of the margin $\Delta$ we have $p_{(2)}=p_\star - \Delta$. Fix any $j\neq i_\star$. By the row-wise bounds:
\begin{align*}
u_{i_\star} &= \sum_{k=1}^l U_\Sim(\bc_{i_\star}, \bc_k)p_k \ge\ p_\star+\sigma_{\min}(1-p_\star) \\
u_j &= \sum_{k=1}^l U_\Sim(\bc_j, \bc_k)p_k  \le \sigma_{\max}+(1-\sigma_{\max})p_j \,.
\end{align*}

Therefore
\begin{align*}
u_{i_\star}-u_j
&\ge p_\star+\sigma_{\min}(1-p_\star)\;-\;\Big(\sigma_{\max}+(1-\sigma_{\max})\,p_j\Big) \\
&\ge p_\star+\sigma_{\min}(1-p_\star)\;-\;\Big(\sigma_{\max}+(1-\sigma_{\max})(p_\star-\Delta)\Big) \\
&= -(\sigma_{\max}-\sigma_{\min})(1-p_\star)\;+\;(1-\sigma_{\max})\,\Delta.
\end{align*}
If \cref{eq:agreement-condition} holds, then the right-hand side is nonnegative, so
$u_{i_\star}\ge u_j$ for every $j\neq i_\star$. Hence $i_\star\in\arg\max_i u_i$ and the
two decoders agree.
\end{proof}


\subsection{Worst-case regrets depending on $m$}
\label{app:worst_case_depending_on_m}



\thmbitwisevectorwise*

\begin{proof}
\textbf{(a) HR@1.)}
The equality $\sup \mathcal{R}_\HRAtOne^{\mathrm{vs}\,\Bit}=\tfrac{1}{2}$ is a known tight result; see, e.g., the analysis of subset accuracy regret when predicting with Hamming-optimal rules in \citet{dembczynski2012label}. 

\medskip
\noindent\textbf{(b) Tanimoto and (c) Cosine: lower bounds $\ge \tfrac{1}{2}$.}
Fix $\varepsilon>0$ small and define the following conditional distribution on $\mathcal{Y}$:
\[
p_\varepsilon(\boldsymbol{y}\mid \boldsymbol{x}) \;=\;
\begin{cases}
0.5-\varepsilon, & \text{if }\boldsymbol{y}=\boldsymbol{c}_1,\\
0.5-(2m-3)\varepsilon, & \text{if }\boldsymbol{y}=\bar{\boldsymbol{c}}_1,\\
2\varepsilon, & \text{if }y_1=0\text{ and } d_H(\boldsymbol{y},\bar{\boldsymbol{c}}_1)=1,\\
0, & \text{otherwise,}
\end{cases}
\]
where $\boldsymbol{c}_1=(1,0,\dots,0)$, $\bar{\boldsymbol{c}}_1=(0,1,\dots,1)$, and $d_H$ is Hamming distance. For $m\ge 3$ this is a valid distribution (nonnegative and summing to $1$) and every $\boldsymbol{y}$ in its support has $\|\boldsymbol{y}\|>0$.

\emph{Step 1: The Hamming-optimal predictor is the zero vector.}
Let $p_\varepsilon^{(i)}(1):=\mathbb{P}(Y_i=1\mid \boldsymbol{x})$ under $p_\varepsilon$. Then
\[
p_\varepsilon^{(1)}(1)=0.5-\varepsilon,\qquad
p_\varepsilon^{(j)}(1)=(0.5-(2m-3)\varepsilon)+(m-2)\cdot 2\varepsilon=0.5-\varepsilon \quad (j\ge 2).
\]
Thus $p_\varepsilon^{(i)}(0)=0.5+\varepsilon$ for all $i$, and the Bayes rule for $U_\Bit$ is
\[
\boldsymbol{y}_\Bit^\star(\boldsymbol{x})=\boldsymbol{0}.
\]

\emph{Step 2: A candidate predictor with large expected utility.}
Consider $\hat{\boldsymbol{y}}=\bar{\boldsymbol{c}}_1$.

\underline{Tanimoto.} For $\boldsymbol{y}=\boldsymbol{c}_1$ the intersection is $0$, so $U_\Tan(\boldsymbol{c}_1,\hat{\boldsymbol{y}})=0$. For $\boldsymbol{y}=\bar{\boldsymbol{c}}_1$ we have $U_\Tan=1$. If $\boldsymbol{y}$ is a neighbor of $\bar{\boldsymbol{c}}_1$ with $y_1=0$ and $d_H(\boldsymbol{y},\bar{\boldsymbol{c}}_1)=1$, then $\|\boldsymbol{y}\wedge\hat{\boldsymbol{y}}\|_0=m-2$ and $\|\boldsymbol{y}\vee\hat{\boldsymbol{y}}\|_0=m-1$, hence $U_\Tan=(m-2)/(m-1)$. Therefore
\begin{align*}
\mathbb{E}\big[U_\Tan(\boldsymbol{Y},\hat{\boldsymbol{y}})\mid \boldsymbol{x}\big]
&= 0\cdot \big(0.5-\varepsilon\big) + 1\cdot \big(0.5-(2m-3)\varepsilon\big) + (m-1)\cdot 2\varepsilon\cdot \frac{m-2}{m-1} \\
&= 0.5 -\varepsilon.
\end{align*}
Since $U_\Tan(\boldsymbol{y},\boldsymbol{0})=0$ for all $\boldsymbol{y}\neq \boldsymbol{0}$, we have
\[
\mathbb{E}\big[U_\Tan(\boldsymbol{Y},\boldsymbol{y}_\Bit^\star(\boldsymbol{x}))\mid \boldsymbol{x}\big]=0,
\]
and consequently
\[
\mathbb{E}\big[U_\Tan(\boldsymbol{Y},\hat{\boldsymbol{y}})\mid \boldsymbol{x}\big]
-\mathbb{E}\big[U_\Tan(\boldsymbol{Y},\boldsymbol{y}_\Bit^\star(\boldsymbol{x}))\mid \boldsymbol{x}\big]
= 0.5-\varepsilon.
\]
Since $h^\star_\Tan$ maximizes expected Tanimoto, its performance is at least that of $\hat{\boldsymbol{y}}$, hence
$\sup \mathcal{R}_\Tan\ge 0.5-\varepsilon$. Letting $\varepsilon\downarrow 0$ yields $\sup \mathcal{R}_\Tan^{\mathrm{vs}\,\Bit} \ge \tfrac{1}{2}$.

\underline{Cosine.} With $\hat{\boldsymbol{y}}=\bar{\boldsymbol{c}}_1$ we have:
\[
U_\Cos(\boldsymbol{c}_1,\hat{\boldsymbol{y}})=0,\qquad
U_\Cos(\bar{\boldsymbol{c}}_1,\hat{\boldsymbol{y}})=1,\qquad
U_\Cos(\boldsymbol{y},\hat{\boldsymbol{y}})=\frac{m-2}{\sqrt{m-2}\,\sqrt{m-1}}=\sqrt{\frac{m-2}{m-1}}
\]
for any neighbor $\boldsymbol{y}$ of $\bar{\boldsymbol{c}}_1$ as above. Therefore
\begin{align*}
\mathbb{E}\big[U_\Cos(\boldsymbol{Y},\hat{\boldsymbol{y}})\mid \boldsymbol{x}\big]
&= 0\cdot \big(0.5-\varepsilon\big) + 1\cdot \big(0.5-(2m-3)\varepsilon\big)
+ (m-1)\cdot 2\varepsilon\cdot \sqrt{\frac{m-2}{m-1}} \\
&= 0.5 + \varepsilon\!\left[-(2m-3) + 2(m-1)\sqrt{1-\frac{1}{m-1}}\right].
\end{align*}
Since $\sqrt{1-\alpha}\ge 1-\alpha$ for $\alpha\in[0,1]$, the bracketed term is $\ge -1$, hence
\[
\mathbb{E}\big[U_\Cos(\boldsymbol{Y},\hat{\boldsymbol{y}})\mid \boldsymbol{x}\big]\ \ge\ 0.5-\varepsilon.
\]
Again $U_\Cos(\boldsymbol{y},\boldsymbol{0})=0$ for all $\boldsymbol{y}$, so
\[
\mathbb{E}\big[U_\Cos(\boldsymbol{Y},\boldsymbol{y}_\Bit^\star(\boldsymbol{x}))\mid \boldsymbol{x}\big]=0,
\]
and therefore
\[
\mathbb{E}\big[U_\Cos(\boldsymbol{Y},\boldsymbol{y}_\Cos^\star(\boldsymbol{x}))\mid \boldsymbol{x}\big]
-\mathbb{E}\big[U_\Cos(\boldsymbol{Y},\boldsymbol{y}_\Bit^\star(\boldsymbol{x}))\mid \boldsymbol{x}\big]
\ \ge\ 0.5-\varepsilon.
\]
Letting $\varepsilon\downarrow 0$ gives $\sup \mathcal{R}_\Cos^{\mathrm{vs}\,\Bit} \ge \tfrac{1}{2}$.
\end{proof}

Remarks: (i) The equality $\sup \mathcal{R}_\HRAtOne^{\mathrm{vs}\,\Bit}=\tfrac{1}{2}$ is tight; see \citet{dembczynski2012label}. For Tanimoto and Cosine, the bounds are \emph{lower} bounds obtained by explicit constructions; sharper (possibly tight) constants are an interesting open direction. 
(ii) The proof uses the convention $U_\Cos(\boldsymbol{y},\boldsymbol{0})=0$ and similarly for $U_\Tan$; the construction avoids $\boldsymbol{y}=\boldsymbol{0}$ so denominators remain well-defined.

\begin{theorem}[Worst-case regret when decoding by HR@1] For $m > 3$ it holds that:
\label{thm:regret_vs_hr1}
\[
\qquad
\sup_{p(\cdot | \bx)} \mathcal{R}_\Tan^{\mathrm{vs\,HR@1}}
\ \ge\ \frac{m-1}{m+2},
\qquad
\sup_{p(\cdot | \bx)} \mathcal{R}_\Cos^{\mathrm{vs\,HR@1}}
\ \ge\ \frac{\sqrt{m(m-1)}}{m+2}
\,.
\]
\end{theorem}

\begin{proof}
Fix $\varepsilon>0$ small and let, for $\alpha:=\tfrac{1}{m+2}-\varepsilon$,
\[
p_\varepsilon(\boldsymbol{y}\mid \boldsymbol{x})
=\begin{cases}
\tfrac{1}{m+2} + (m+1)\varepsilon, & \boldsymbol{y}=\boldsymbol{0}_m,\\
\alpha, & \boldsymbol{y}=\boldsymbol{1}_m \ \ \text{or}\ \ d_H(\boldsymbol{y},\boldsymbol{0}_m)=m-1,\\
0, & \text{otherwise,}
\end{cases}
\]
where $d_H$ is Hamming distance. This is a valid distribution since the support has
$(m+1)$ points with mass $\alpha$ and one point with mass $\tfrac{1}{m+2}+(m+1)\varepsilon$, summing to $1$.

\medskip
\noindent\textbf{Step 1: The decoders $\boldsymbol{y}_\HRAtOne^\star$ and $\boldsymbol{y}_\Bit^\star$ under $p_\varepsilon$.}
The joint mode is $\boldsymbol{0}_m$ because
$p_\varepsilon(\boldsymbol{0}_m)=\tfrac{1}{m+2}+(m+1)\varepsilon
>\alpha=p_\varepsilon(\boldsymbol{y})$ for every other support point. Hence
$\boldsymbol{y}_\HRAtOne^\star(\boldsymbol{x})=\boldsymbol{0}_m$.
For marginals,
\[
\mathbb{P}(Y_i=1\mid \boldsymbol{x})
= p_\varepsilon(\boldsymbol{1}_m) + (m-1)\,p_\varepsilon(\text{“exactly one zero not at $i$”})
= \alpha + (m-1)\alpha
= m\alpha,
\]
so $\mathbb{P}(Y_i=0\mid \boldsymbol{x})=1-m\alpha=\tfrac{2}{m+2}+m\varepsilon$.
For $m\ge 3$ and $\varepsilon$ sufficiently small, $m\alpha>\tfrac{1}{2}$, thus
$\boldsymbol{y}_\Bit^\star(\boldsymbol{x})=\boldsymbol{1}_m$.

\medskip
\noindent\textbf{Step 2: Bitwise accuracy (tight equality).}
Since $U_\Bit(\boldsymbol{y},\boldsymbol{1}_m)=\tfrac{\|\boldsymbol{y}\|_0}{m}$
and $U_\Bit(\boldsymbol{y},\boldsymbol{0}_m)=1-\tfrac{\|\boldsymbol{y}\|_0}{m}$,
\[
\mathbb{E}\big[U_\Bit(\boldsymbol{Y},\boldsymbol{1}_m)\mid \boldsymbol{x}\big]
= m\alpha,\qquad
\mathbb{E}\big[U_\Bit(\boldsymbol{Y},\boldsymbol{0}_m)\mid \boldsymbol{x}\big]
= 1-m\alpha.
\]
Thus the (pointwise) regret incurred by HR@1 equals
\[
m\alpha-(1-m\alpha)=2m\alpha-1
= \frac{m-2}{m+2} - 2m\varepsilon.
\]
Letting $\varepsilon\downarrow 0$ shows
$\sup \mathcal{R}_\Bit^{\mathrm{vs\,HR@1}}\ge \tfrac{m-2}{m+2}$.
The reverse inequality (hence equality) is known to be tight; see \citet{dembczynski2012label}.

\medskip
\noindent\textbf{Step 3: Tanimoto (lower bound).}
Under the conventions in the definition,
$U_\Tan(\boldsymbol{0}_m,\boldsymbol{0}_m)=1$ and
$U_\Tan(\boldsymbol{y},\boldsymbol{0}_m)=0$ for $\boldsymbol{y}\neq \boldsymbol{0}_m$,
so
\[
\mathbb{E}\big[U_\Tan(\boldsymbol{Y},\boldsymbol{y}_\HRAtOne^\star)\mid \boldsymbol{x}\big]
= p_\varepsilon(\boldsymbol{0}_m)
= \tfrac{1}{m+2}+(m+1)\varepsilon.
\]
Consider $\hat{\boldsymbol{y}}=\boldsymbol{1}_m$.
Then
$U_\Tan(\boldsymbol{1}_m,\hat{\boldsymbol{y}})=1$ and,
for any $\boldsymbol{y}$ with exactly one zero,
$U_\Tan(\boldsymbol{y},\hat{\boldsymbol{y}})=\tfrac{m-1}{m}$, while
$U_\Tan(\boldsymbol{0}_m,\hat{\boldsymbol{y}})=0$.
Hence
\[
\mathbb{E}\big[U_\Tan(\boldsymbol{Y},\hat{\boldsymbol{y}})\mid \boldsymbol{x}\big]
= \alpha\cdot 1 + m\alpha\cdot \frac{m-1}{m}
= m\alpha
= \frac{m}{m+2} - m\varepsilon.
\]
Since $\boldsymbol{y}_\Tan^\star$ maximizes expected Tanimoto,
\[
\sup_{p(\cdot | \bx)} \mathcal{R}_\Tan^{\mathrm{vs\,HR@1}}
\ \ge\ m\alpha - \Big(\tfrac{1}{m+2}+(m+1)\varepsilon\Big)
= \frac{m-1}{m+2} - (2m+1)\varepsilon.
\]
Let $\varepsilon\downarrow 0$ to obtain
$\sup \mathcal{R}_\Tan^{\mathrm{vs\,HR@1}}\ge \tfrac{m-1}{m+2}$.

\medskip
\noindent\textbf{Step 4: Cosine (lower bound).}
With $\boldsymbol{y}_\HRAtOne^\star=\boldsymbol{0}_m$ and our convention
$U_\Cos(\boldsymbol{0}_m,\boldsymbol{0}_m)=1$ and $0$ otherwise,
\[
\mathbb{E}\big[U_\Cos(\boldsymbol{Y},\boldsymbol{y}_\HRAtOne^\star)\mid \boldsymbol{x}\big]
= p_\varepsilon(\boldsymbol{0}_m)
= \tfrac{1}{m+2}+(m+1)\varepsilon.
\]
Take again $\hat{\boldsymbol{y}}=\boldsymbol{1}_m$.
Then $U_\Cos(\boldsymbol{1}_m,\hat{\boldsymbol{y}})=1$ and,
for any $\boldsymbol{y}$ with exactly one zero,
\[
U_\Cos(\boldsymbol{y},\hat{\boldsymbol{y}})
=\frac{m-1}{\sqrt{m}\,\sqrt{m-1}}
=\sqrt{\frac{m-1}{m}}.
\]
Therefore
\begin{align*}
\mathbb{E}\big[U_\Cos(\boldsymbol{Y},\hat{\boldsymbol{y}})\mid \boldsymbol{x}\big]
&= \alpha\cdot 1 + m\alpha\cdot \sqrt{\frac{m-1}{m}}
= \alpha\Big(1+\sqrt{m(m-1)}\Big).
\end{align*}
By optimality of $\boldsymbol{y}_\Cos^\star$,
\[
\sup_{p(\cdot | \bx)} \mathcal{R}_\Cos^{\mathrm{vs\,HR@1}}
\ \ge\ \alpha\Big(1+\sqrt{m(m-1)}\Big) - \Big(\tfrac{1}{m+2}+(m+1)\varepsilon\Big)
= \frac{\sqrt{m(m-1)}}{m+2} - \big(\sqrt{m(m-1)}+m+2\big)\varepsilon.
\]
Letting $\varepsilon\downarrow 0$ yields
$\sup \mathcal{R}_\Cos^{\mathrm{vs\,HR@1}}\ge \tfrac{\sqrt{m(m-1)}}{m+2}$, as claimed.
\end{proof}

Remarks: (i) The equality for $U_\Bit$ is tight; see \citet{dembczynski2012label}.
(ii) The lower bounds for Tanimoto and Cosine come from an explicit adversarial family $p_\varepsilon$; sharper constants (or tightness) are an open question.
(iii) If one instead sets $U_\Tan(\boldsymbol{0},\boldsymbol{0})=0$ and/or $U_\Cos(\boldsymbol{0},\boldsymbol{0})=0$, the bounds shift by $-\tfrac{1}{m+2}$ accordingly; we adopted the “empty equals empty” similarity convention to match common Jaccard practice and to align the constants with the statement above.
(iv) For $m=2$, the construction yields $\boldsymbol{y}_\Bit=\boldsymbol{0}$ as well, and the bitwise gap collapses to $0=(m-2)/(m+2)$, consistent with the theorem.

\section{From Bayes regret to finite-sample predictors}
\label{sec:generalization}

The regret analysis in \cref{sec:regret_analysis} is carried out at the \emph{Bayes level}: we compare decision rules that are optimal under different utilities $U$ and $V$ (e.g.\ HR@1, Tanimoto, cosine, bitwise accuracy). In practice, however, fingerprint predictors are trained on a \emph{finite} sample by minimizing an empirical loss. The resulting decision rule therefore deviates from the Bayes-optimal rule.

In this section we separate these two phenomena:
\begin{itemize}
    \item a purely \emph{intrinsic Bayes-level mismatch} between using $U$-optimal or $V$-optimal decisions (captured by the regret bounds in \cref{sec:regret_analysis} and \cref{app:regret_proofs}), and
    \item a \emph{finite-sample generalization component} arising from the estimation error of the fingerprint predictor trained by empirical risk minimization (ERM).
\end{itemize}
Formally, we obtain an exact decomposition of the $U$-regret into a Bayes term plus a residual term, and we show how the residual can be controlled via standard learning-theoretic bounds on the surrogate loss.

We recall the probabilistic framework from \cref{sec:contrastive_as_retrieval}. Let $\mathcal{F}$ be a hypothesis class of fingerprint predictors $
f : \mathcal{X} \to [0,1]^m $,
for instance the neural networks of the form $f(\bx)=\sigma(\boldsymbol{W}_\text{out} E(\bx))$ described in \cref{sec:suppmethodsmsmole}. Training is performed by minimizing an empirical loss
\[
\widehat{L}_n(f)
:=\frac{1}{n}\sum_{i=1}^n \ell\big(f(\bx^{(i)}),\by^{(i)}\big),
\]
for some surrogate loss $\ell:[0,1]^m\times\{0,1\}^m\to[0,1]$ (e.g.\ BCE, focal, IoU, contrastive loss). The corresponding population (true) risk is
\[
L(f)
:= \mathbb{E}\big[\ell(f(\bX),\bY)\big],
\qquad
L^\star := \inf_{f\in\mathcal{F}} L(f).
\]
Remark that in this section we will use the notation $\bX$ and $\bY$ when mass spectra and labels are random variables (both can be random here, unlike \cref{sec:regret_analysis} and \cref{app:regret_proofs}, where only the labels were random variables). 
Given a trained predictor $f$, we obtain a decision rule by a decoding map
\[
\pi_f : \mathcal{X} \to \mathcal{C}(\bx),
\qquad
\bx \mapsto \pi_f(\bx),
\]
which depends on the modeling choice:
\begin{itemize}
    \item for listwise/contrastive models, $\pi_f(\bx)$ typically selects the candidate with largest score $g_f(\bx,\mathbf{c})$, cf.\ \cref{eq:fp-cos,eq:emb-cos,eq:fp-mlp,eq:emb-mlp};
    \item for bitwise or vectorwise models, $\pi_f(\bx)$ may select the candidate that maximizes a similarity to $f(\bx)$ (Tanimoto, cosine, etc.).
\end{itemize}
The induced rule $\pi_f$ can then be evaluated under any utility $U:\mathcal{Y}\times\mathcal{Y}\to[0,1]$ (HR@1, HR@k, Tanimoto, cosine, bitwise accuracy, \ldots).
Recall the definition of a Bayes-optimal decision in Eq.\ \cref{eq:bayes_opt}, then the \emph{conditional $U$-regret} of a decision rule $\pi$ at $\bx$ is
\[
\mathcal{R}_U(\pi\mid\bx)
:= \mathbb{E}\big[U(\bY,\by_U^\star(\bx)) - U(\bY,\pi(\bx)) \mid \bx\big],
\]
and the global regret is $\mathcal{R}_U(\pi) := \mathbb{E}_{\bX}[\mathcal{R}_U(\pi\mid\bX)]$. We also recall the Bayes-level (plug-in) regret of decoding with $V$ when evaluated under $U$:
\[
\mathcal{R}_U^{\mathrm{vs}\,V}(\bx)
:= \mathbb{E}\big[U(\bY,\by_U^\star(\bx)) - U(\bY,\by_V^\star(\bx)) \mid \bx\big],
\]
which has been bounded explicitly for several $(U,V)$ in \cref{sec:contrastive_as_retrieval} and \cref{app:worst_case_depending_on_m}.

\begin{lemma}[Exact decomposition]
\label{lem:regret_decomposition_exact}
For any decision rule $\pi$ and utilities $U,V$,
\begin{equation}
\label{eq:regret_exact_decomposition}
\mathcal{R}_U(\pi)
=
\mathbb{E}_{\bX}\big[\mathcal{R}_U^{\mathrm{vs}\,V}(\bX)\big]
+
\mathbb{E}_{\bX}\big[\Delta_U^{(V)}(\pi\mid\bX)\big],
\end{equation}
where
\[
\Delta_U^{(V)}(\pi\mid\bx)
:= \mathbb{E}\big[U(\bY,\by_V^\star(\bx)) - U(\bY,\pi(\bx)) \mid \bx\big].
\]
\end{lemma}

\begin{proof}
For any fixed $\bx$, add and subtract $U(\bY,\by_V^\star(\bx))$:
\begin{align*}
\mathcal{R}_U(\pi\mid\bx)
&= \mathbb{E}\big[U(\bY,\by_U^\star(\bx)) - U(\bY,\pi(\bx)) \mid \bx\big] \\
&= \mathbb{E}\big[U(\bY,\by_U^\star(\bx)) - U(\bY,\by_V^\star(\bx)) \mid \bx\big]
 + \mathbb{E}\big[U(\bY,\by_V^\star(\bx)) - U(\bY,\pi(\bx)) \mid \bx\big] \\
&= \mathcal{R}_U^{\mathrm{vs}\,V}(\bx) + \Delta_U^{(V)}(\pi\mid\bx).
\end{align*}
Taking expectations over $\bX$ yields \cref{eq:regret_exact_decomposition}.
\end{proof}

The decomposition in \cref{lem:regret_decomposition_exact} is purely algebraic and does \emph{not} rely on any inequality between $U$ and $V$:
\begin{itemize}
    \item the first term, $\mathbb{E}[\mathcal{R}_U^{\mathrm{vs}\,V}(\bX)]$, is an intrinsic \emph{Bayes-level mismatch} between decoding with $U$-optimal and $V$-optimal rules;
    \item the second term, $\mathbb{E}[\Delta_U^{(V)}(\pi\mid\bX)]$, measures how much the learned rule $\pi$ deviates from the $V$-Bayes rule, as seen under the utility $U$. In theory this term can be negative, which happens when the learned rule $\pi$ with surrogate loss based on $V$ is closer to the Bayes-optimal than the $V$-Bayes rule. However, when $\ell$ is a proper or calibrated surrogate for the utility $V$, this term will be positive with very high probability. 
\end{itemize}
Our regret theorems in \cref{sec:contrastive_as_retrieval} and \cref{app:worst_case_depending_on_m} provide explicit bounds on the first (Bayes) term for concrete $(U,V)$ pairs (e.g.\ HR@1 vs.\ Tanimoto, HR@1 vs.\ bitwise accuracy). The second term is where generalization and the choice of surrogate loss enter. Numerous papers in statistical learning theory have suggested generalization bounds for the second term. Many of these bounds are independent of the particular utilities $U$ and $V$; they only depend on the complexity of the hypothesis class $\mathcal{F}$ (see e.g.\ bounds based on empirical Rademacher complexity). Let $\hat{f}_n$ be any empirical risk minimizer in $\mathcal{F}$, i.e.,
\[
\widehat{L}_n(\hat{f}_n)
= \inf_{f\in\mathcal{F}} \widehat{L}_n(f).
\]
Then, classical results in learning theory show that the \emph{excess surrogate risk} $L(\hat{f}_n)-L^\star$ scales with the complexity of $\mathcal{F}$ and decreases as $n$ grows~\citep{bartlett2006convexity}. This is the standard generalization term for the fingerprint predictor.


For concrete surrogate--metric pairs (e.g.\ BCE for bitwise accuracy, IoU loss for Tanimoto), existing calibration results relate \emph{excess surrogate risk} to \emph{metric-specific regret}. Such calibration inequalities are standard in the statistical learning literature; for instance, logistic/BCE loss is calibrated for $0/1$ accuracy \citep{bartlett2006convexity}, and similar results hold for top-$k$ probabilities \citep{yang2020consistency}. Combining these calibration inequalities with the exact decomposition Eq.~\ref{eq:regret_exact_decomposition} yields metric-dependent finite-sample bounds of the form:
\[
\mathcal{R}_U(\pi_{\hat{f}_n})
=
\underbrace{\text{Bayes mismatch between $U$ and $V$}}_{\text{analyzed in \cref{sec:contrastive_as_retrieval}}}
\;+\;
\underbrace{\text{generalization term that shrinks with $n$}}_{\text{controlled by }L(\hat{f}_n)-L^\star}.
\]

We intentionally leave this second part in a general form: the exact constants and functional form depend on $(U,V,\ell)$ and on the decoding $\pi_f$. In summary the second term depends on:
\begin{itemize}
    \item how well the surrogate-trained predictor $\hat{f}_n$ approximates the $V$-Bayes decision rule, and
    \item how changes in decisions (from $\by_V^\star(\bx)$ to $\pi_{\hat{f}_n}(\bx)$) affect the utility $U$.
\end{itemize}

\section{Supplementary methods} \label{sec:suppmethodsmsmole}

\paragraph{MLP architecture.}
The full fingerprint prediction model, denoted by $f$, produces predictions $\hat{\boldsymbol{y}}\in[0,1]^m$ according to:
\begin{equation*}
\hat{\boldsymbol{y}}^{(i)} = f(\boldsymbol{x}^{(i)}) = \sigma(\boldsymbol{W}_\text{out} E(\boldsymbol{x}^{(i)})),
\end{equation*}
where $\sigma$ denotes the sigmoid operation, and $\boldsymbol{W}_\text{out} \in \mathbb{R}^{m \times d}$ denotes the learnable weights of the last linear layer.
All other neural network layers are included in $E$.
In all experiments, $E$ uses the same MLP architecture consisting of 3 fully-connected layers.
In between each layer, a GeLU activation~\cite{hendrycks2016gaussian}, dropout~\cite{srivastava2014dropout} (using a rate of 0.25) and LayerNorm operation~\cite{ba2016layer} is placed, in that order.
The first layer maps the 10050-dimensional space to 1024 dimensions, which is also the hidden dimensionality used in all subsequent layers.
More formally, let $B(\cdot)=\text{LayerNorm}(\text{Dropout}(\text{GeLU}(\cdot), 0.25))$.
the neural network embedder $E$ is, then, described fully as:
\begin{equation*}
	E(\boldsymbol{x}^{(i)})=\boldsymbol{W}_4 \cdot B(\boldsymbol{W}_3 \cdot B(\boldsymbol{W}_2 \cdot B(\boldsymbol{W}_1\boldsymbol{x}^{(i)})))
\end{equation*}
with $\boldsymbol{W}_1\in\mathbb{R}^{1024\times 10050}$, $\boldsymbol{W}_{\{2,3\}}\in\mathbb{R}^{1024\times 1024}$, and $\boldsymbol{W}_4\in\mathbb{R}^{1024\times 1024}$, the learnable weight matrices of the linear layers (with biases omitted for simplicity).

Models were trained using one of the eight loss functions explained in \cref{sec:loss_fp}: (1) BCE, (2) FL, (3) Cosine Loss, (4) IoU Loss, (5) Contrastive FP-Cos, (6) Contrastive Emb-Cos, (7) Contrastive Fp-MLP, and (8) Contrastive Emb-MLP.
For the contrastive MLP variants, additional layers are also fitted that either take (1) the predicted and true fingerprint as input ("Fp-MLP"), or (2) an embedding of the spectrum and candidate molecule as input ("Emb-MLP").
In the former case, the MLP is of the form:
\begin{equation*}
	\text{MLP}(\hat{\boldsymbol{y}}, \boldsymbol{c}_j) = \boldsymbol{W}_3 \cdot B(\boldsymbol{W}_2 \cdot B(\boldsymbol{W}_1 (\hat{\boldsymbol{y}}~\|~\boldsymbol{c}_j~\|~ \hat{\boldsymbol{y}}\odot\boldsymbol{c}_j) ) ).
\end{equation*}
The input to this similarity-learning MLP, hence, uses a concatenation of (1) the predicted fingerprint $\hat{\boldsymbol{y}}\in[0,1]^{4096}$, (2) a candidate fingerprint $\boldsymbol{c}_j\in\{0,1\}^{4096}$, and (3) their element-wise product $\hat{\boldsymbol{y}}\odot\boldsymbol{c}_j \in[0,1]^{4096}$.
The MLP uses learnable linear layers with weight matrices (biases omitted for simplicity): $\boldsymbol{W}_1\in\mathbb{R}^{1536\times 12288}$, $\boldsymbol{W}_2\in\mathbb{R}^{768\times 1536}$, and $\boldsymbol{W}_3\in\mathbb{R}^{1\times 768}$.
The function $B$ denotes the same block as in the main fingerprint-producing MLP: $B(\cdot)=\text{LayerNorm}(\text{Dropout}(\text{GeLU}(\cdot), 0.25))$.
In the latter case ("Emb-MLP"), the MLP is of a similar form.
Let us denote the spectrum embedding and candidate embedding as $\boldsymbol{e}_\text{sp}=\boldsymbol{W}_\text{sp}E(\boldsymbol{x})$ and $\boldsymbol{e}_{\text{c},j} = \boldsymbol{W}_\text{c}\boldsymbol{c}_j$.
The MLP is the "Emb-MLP" contrastive model is, then:
\begin{equation*}
	\text{MLP}(\boldsymbol{e}_\text{sp}, \boldsymbol{e}_{\text{c},j}) = \boldsymbol{W}_3 \cdot B(\boldsymbol{W}_2 \cdot B(\boldsymbol{W}_1 (\boldsymbol{e}_\text{sp}~\|~\boldsymbol{e}_{\text{c},j}~\|~ \boldsymbol{e}_\text{sp}\odot\boldsymbol{e}_{\text{c},j}) ) ),
\end{equation*}
where $\boldsymbol{W}_1\in\mathbb{R}^{96\times 768}$, $\boldsymbol{W}_2\in\mathbb{R}^{48\times 96}$, and $\boldsymbol{W}_3\in\mathbb{R}^{1\times 48}$ denote the learnable linear weight matrices (biases omitted for simplicity).

\paragraph{MLP main experiments.}
All models were trained with the Adam optimizer in float-32 precision using a batch size of 64 for a maximum of 50 epochs.
Gradients were clipped to a norm of 1.
A variable learning rate (however, fixed during one run) was used depending on the hyperparameter configuration (see below).
After every 1000 training steps, validation performance was checked in terms of: (1) HR@1, (2) HR@5, (3) HR@20, and (4) Tanimoto similarity.
For each of these metrics, the checkpoint was saved for when that metric was optimal (i.e., maximal).
For every model, the appropriate checkpoint was used to compute final validation and test performance.
For example, reported validation and test HR@20 scores were always derived from model checkpoints with optimal validation HR@20s during their training runs.
This was performed because the convergence in terms of Tanimoto similarity may not coincide with the convergence in terms of retrieval.
For bitwise and vectorwise loss functions, hit rates were computed (and checkpoints made) for both retrieval settings.
For contrastive loss functions---as these models use the candidates during training---separate model runs were performed for the two retrieval settings.
Each model was trained on a single A100 GPU with a compute budget of 16 hours.

For every loss function, five different learning rates were tested: $\{0.00005, 0.00007$, $0.0001$, $0.0003$, $0.0005\}$.
For each of these possible learning rates, 5 models were trained.
The best learning rate was then chosen based on the highest average HR@20 that learning rate achieved.
All other metrics are then reported using the checkpoints derived from those 5 models.
For the focal loss, the $\gamma$ hyperparameter was additionally tuned in the same manner, but using a fixed learning rate of $0.00007$ (a value chosen because it was the optimal value for the BCE loss). Results are in Appendix~\cref{tab:gamma_tune}.
For the contrastive loss function, the same was performed for different values of $\tau$, results are in Appendix~\cref{tab:tau_tune}.
Upon finding the optimal value of their respective hyperparameters, an identical procedure was performed for finding the optimal learning rate and reporting results as with the other models.

\paragraph{Transformer architecture.}
All tuning was performed using MLP models.
Transformer experiments adopted all the same hyperparameters, only differing in the $E$ spectrum embedder.
For each spectrum, the input consists of the max 128 highest peaks.
A transformer with $4$ encoder blocks was used, each block using $d=256$ as hidden dimensionality, with $8$ attention heads, a feedforward dimension of $1024$, and a dropout rate of $0.25$.
All other architectural hyperparameters follow the defaults of the \texttt{depthcharge} package.
Transformer peak tokens $\boldsymbol{Z}=[\boldsymbol{z}_1, \dots, \boldsymbol{z}_m]$ were aggregated by an attention aggregator:
\begin{equation*}
    \boldsymbol{z}_\text{agg} = \sum_{i=1}^m \boldsymbol{z}_i a_i\,
\end{equation*}
where $a_i$ is an attention weight derived from $\boldsymbol{a} = \text{softmax}(\boldsymbol{Z}\boldsymbol{W}_\text{att})$, with learnable weight matrix $\boldsymbol{W}_\text{att} \in\mathbb{R}^{d\times 1}$.
After this layer, the same procedure as above (using a final fingerprint bit prediction layer and sigmoid) was used.

\paragraph{Enveda-180 experiments.}
All Enveda-180 experiments adopt the hyperparameters from our MassSpecGym experiments.
The only difference in experimental set-up is that models were run for a maximum of 20 epochs instead of 50.

\section{Figures and Tables supporting the results section}
\label{sec:extra_results}

\begin{figure}[H]
	\centering
	\includegraphics[width=0.5\linewidth]{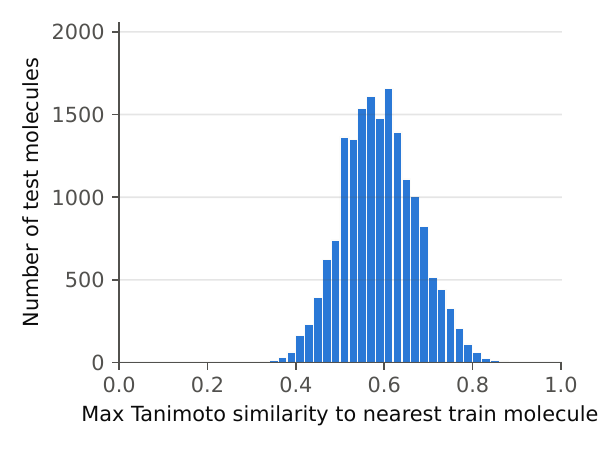}
	\caption{
		Max Tanimoto of test molecules to training compounds of our Enveda-180 split. No test molecules have near identical fingerprint counterparts in the training set.
        }
	\label{fig:envedaleakage}
\end{figure}

\begin{figure}[H]
	\centering
	\includegraphics[width=0.9\linewidth]{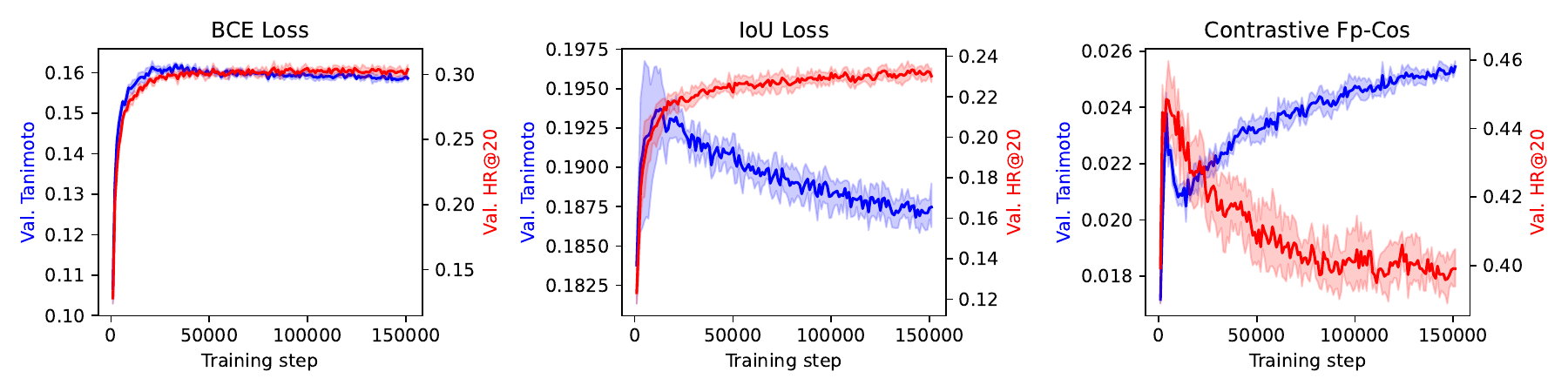}
	\caption{
		Validation performance (in terms of Average Tanimoto, blue, and HR@20, red) over training time of MLP models on MassSpecGym.
        Curves are shown for the BCE Loss, IoU Loss, and Contrastive FP-Cos, respectively.
        Note that the Contrastive FP-Cos curve is obtained using a temperature of 1/256.
        As such, the model is tuned for better hit rates (see \cref{tab:tau_tune}).
        In addition, note that our experimental setup accounts for different optima in time w.r.t.\ different metrics, as final scores are reported using the checkpoint for which each metric was optimal.
        }
	\label{fig:val_over_time}
\end{figure}

\begin{table}[H]
    \centering
    \setlength{\tabcolsep}{4pt}
    \caption{
        Focal loss $\gamma$ MLP tuning results~\citep{bushuiev2024massspecgym} on MassSpecGym.
        All performances are validation scores.
        Numbers report averages and standard deviations over 5 model runs.
        All 20 models used in this table were trained with a learning rate of $7\text{e-}5$.
        Retrieval scores are those using equal mass candidates.
        The best hyperparameter value (bold) was chosen based on retrieval scores.
    }
    \resizebox{0.55\linewidth}{!}{
        \begin{tabular}{ccccc}
        \\ \toprule Focal $\gamma$ & Tanimoto (IoU) & HR@1 & HR@5  & HR@20 \\ \midrule
        $1.0$                     & $\mathbf{16.27\pm0.07}$    & $08.94\pm0.18$ & $18.14\pm0.37$ & $32.45\pm0.33$ \\
        $\mathbf{2.0}$                     & $16.20\pm0.06$     & $\mathbf{09.20\pm0.08}$  & $\mathbf{18.74\pm0.20}$  & $\mathbf{33.05\pm0.16}$ \\
        $5.0$                  & $15.88\pm0.10$     & $08.45\pm0.18$ & $18.10\pm0.25$  & $31.53\pm0.32$ \\
        $10.0$                 & $15.70\pm0.05$     & $06.31\pm0.23$ & $15.00\pm0.33$  & $27.72\pm0.27$ \\ \bottomrule \\
        \end{tabular}
    } \label{tab:gamma_tune}
\end{table}

\begin{table}[H]
    \centering
    \setlength{\tabcolsep}{4pt}
    \caption{
        Contrastive loss temperature $\tau$ MLP tuning results~\citep{bushuiev2024massspecgym}.
        All performances are validation scores.
        Numbers report averages and standard deviations over 5 model runs.
        All models used in this table were trained with a learning rate of $7\text{e-}5$.
        Tanimoto scores are only reported for the "FP-Cos" contrastive model, as this is the only one that can be interpreted as predicting a true fingerprint.
        Retrieval scores are those using equal mass candidates.
        Best hyperparameter values (bold) were chosen based on retrieval scores.
    }
    \resizebox{0.75\linewidth}{!}{
\begin{tabular}{llcccc}
\\ \toprule Contrastive loss type & Temperature $\tau$ & \multicolumn{1}{c}{Tanimoto (IoU)} & \multicolumn{1}{c}{HR@1} & \multicolumn{1}{c}{HR@5} & \multicolumn{1}{c}{HR@20} \\ \midrule
FP-Cos                & $\mathbf{1/256}$          & \multicolumn{1}{l}{$02.67\pm0.03$} & $\mathbf{11.31\pm0.59}$ & $\mathbf{25.51\pm1.02}$ & $\mathbf{45.60\pm1.09}$  \\
                      & $1/64$           & \multicolumn{1}{l}{$04.06\pm0.26$} & $11.17\pm0.18$ & $24.75\pm0.62$ & $44.54\pm0.65$ \\
                      & $1/16$           & \multicolumn{1}{l}{$03.63\pm0.05$} & $09.75\pm0.29$  & $21.42\pm0.25$ & $37.66\pm0.29$ \\
                      & $1/4$            & \multicolumn{1}{l}{$06.66\pm0.05$} & $08.23\pm0.18$  & $18.32\pm0.26$ & $34.13\pm0.23$ \\
                      & $1$              & \multicolumn{1}{l}{$07.19\pm0.04$} & $08.17\pm0.12$  & $18.06\pm0.18$ & $33.48\pm0.15$ \\
                      & $4$              & \multicolumn{1}{l}{$\mathbf{07.38\pm0.04}$} & $08.14\pm0.10$  & $18.20\pm0.19$ & $33.77\pm0.16$ \\ \midrule
Emb-Cos               & $1/256$          & N/A                               & $12.04\pm0.65$ & $26.43\pm0.62$ & $45.43\pm0.62$ \\
                      & $\mathbf{1/64}$  & N/A                               & $\mathbf{12.08\pm0.28}$ & $\mathbf{27.42\pm0.75}$ & $\mathbf{47.44\pm0.78}$ \\
                      & $1/16$           & N/A                               & $10.76\pm0.15$ & $24.29\pm0.23$ & $45.68\pm0.51$ \\
                      & $1/4$            & N/A                               & $06.18\pm0.16$ & $17.64\pm0.39$ & $39.28\pm0.54$ \\
                      & $1$              & N/A                               & $04.08\pm0.09$ & $14.23\pm0.27$ & $33.81\pm0.52$ \\
                      & $4$              & N/A                               & $04.62\pm0.32$ & $14.53\pm0.41$ & $36.58\pm0.86$ \\ \midrule
Fp-MLP                & $1/256$          & N/A                               & $10.66\pm0.58$ & $25.13\pm0.51$ & $45.85\pm0.42$ \\
                      & $1/64$           & N/A                               & $10.81\pm0.58$ & $25.18\pm0.84$ & $46.28\pm1.08$ \\
                      & $\mathbf{1/16}$  & N/A                               & $\mathbf{12.60\pm0.75}$ & $\mathbf{27.84\pm0.51}$ & $\mathbf{47.88\pm0.70}$  \\
                      & $1/4$            & N/A                               & $11.53\pm0.57$ & $26.93\pm0.51$ & $47.21\pm0.36$ \\
                      & $1$              & N/A                               & $12.51\pm0.39$ & $27.48\pm0.60$ & $47.91\pm0.95$ \\
                      & $4$              & N/A                               & $12.75\pm0.19$ & $27.85\pm0.68$ & $47.36\pm0.65$ \\ \midrule
Emb-MLP               & $1/256$          & N/A                               & $10.75\pm0.44$ & $25.05\pm0.62$ & $45.89\pm0.76$ \\
                      & $1/64$           & N/A                               & $10.80\pm0.48$ & $25.66\pm0.89$ & $46.70\pm1.08$  \\
                      & $\mathbf{1/16}$  & N/A                               & $\mathbf{11.50\pm1.29}$ & $\mathbf{26.33\pm1.67}$ & $\mathbf{46.81\pm1.31}$ \\
                      & $1/4$            & N/A                               & $10.85\pm0.77$ & $26.01\pm1.18$ & $46.71\pm0.62$ \\
                      & $1$              & N/A                               & $11.35\pm1.26$ & $25.86\pm1.77$ & $46.46\pm1.74$ \\
                      & $4$              & N/A                               & $11.24\pm1.61$ & $26.12\pm1.89$ & $46.05\pm1.54$ \\ \bottomrule \\
\end{tabular}
    } \label{tab:tau_tune}
\end{table}

\begin{table}[H]
	\centering
	\setlength{\tabcolsep}{4pt}
	\caption{
		MassSpecGym (equal-mass) test set results comparing DreaMS with transformer models trained from-scratch.
		Scores indicate the average and standard deviation over 5 model runs.
        DreaMS experiments fully follow the ones where transformer models are trained from scratch, but a frozen DreaMS spectrum embedding is plugged in as spectrum embedder instead.
	}
	\resizebox{0.80\linewidth}{!}{
	\begin{tabular}{llllll}
		\\ \toprule
		\begin{tabular}[c]{@{}c@{}}Model\end{tabular} & \begin{tabular}[c]{@{}c@{}}Loss function\end{tabular} & \multicolumn{1}{c}{HR@1} & \multicolumn{1}{c}{HR@5} & \multicolumn{1}{c}{HR@20} & \multicolumn{1}{c}{Tanimoto} \\ \midrule
		Transformer from scratch   & IoU Loss      & $09.43 \pm 0.50$ & $ 20.57 \pm 0.72 	$ & $38.80 \pm 0.43 	$ & $22.13 \pm 0.15$	          \\
        & Contr. Emb-Cos            & $\boldsymbol{16.13 ±\pm0.69} 	$ & $\boldsymbol{31.46 \pm 0.70} 	$ & $53.73 \pm 1.41 	$ & N/A          \\ \midrule
		DreaMS   & IoU Loss      & $7.38 \pm 0.98 	$ & $17.21 \pm 0.45 	$ & $34.38 \pm 0.38 	$ & $\boldsymbol{22.79 \pm 0.16}$	          \\
        & Contr. Emb-Cos            & $14.33 \pm 0.95 	$ & $30.14 \pm 1.60 	$ & $\boldsymbol{54.29 \pm 1.04} 	$ & N/A          \\ \bottomrule \\
	\end{tabular}
	} \label{tab:dreams}
\end{table}

\begin{table}[H]
	\centering
	\setlength{\tabcolsep}{4pt}
	\caption{
		Validation retrieval scores of models using vectorwise loss functions on MassSpecGym.
		Only results are shown on the equal mass candidate retrieval setting for MLP models.
		On average, performing retrieval using cosine similarity performs better than using Tanimoto similarity (i.e., IoU).
        Here, IoU retrieval is performed with continuous (i.e., non-binarized) predictions.
		This holds true even for the model trained using the IoU loss.
		Scores indicate the average and standard deviation over 5 model runs.
	}
	\resizebox{0.65\linewidth}{!}{
	\begin{tabular}{lllll}
		\\ \toprule
		\begin{tabular}[c]{@{}c@{}}Model\\ trained using\end{tabular} & \begin{tabular}[c]{@{}c@{}}Evaluate\\ retrieval using\end{tabular} & \multicolumn{1}{c}{HR@1} & \multicolumn{1}{c}{HR@5} & \multicolumn{1}{c}{HR@20} \\ \midrule
		Cosine Loss   & Cosine      & $7.49\pm0.16$	& $15.82\pm0.21$	& $29.39\pm0.21$	          \\
		& Tanimoto (IoU)            & $7.21\pm0.21$	& $15.34\pm0.14$	& $28.93\pm0.25$           \\ \midrule
		IoU Loss      & Cosine      & $4.29\pm0.21$	& $11.00\pm0.18$	& $23.66\pm0.25$           \\
		& Tanimoto (IoU)            		& $3.86\pm0.21$	& $10.45\pm0.20$	& $23.01\pm0.28$           \\ \bottomrule \\
	\end{tabular}
	} \label{tab:retrieval_func}
\end{table}

\begin{figure}[H]
	\centering
	\includegraphics[width=0.35\linewidth]{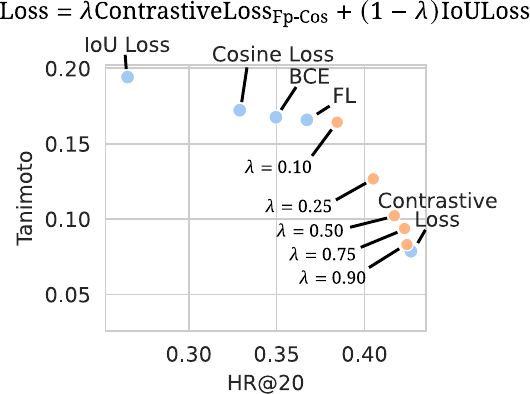}
	\caption{
		Experimental results of MLPs when training models using a weighted combination of loss functions on MassSpecGym.
        In orange, results are shown using a combination of contrastive loss and IoU loss, with 5 different weights $\lambda$.
        Scores indicate the mean over 5 model runs.
        Notably, even small weights on contrastive loss components result in a marked increase in retrieval scores over non-contrastive loss functions.
        }
	\label{fig:inverse_corr2}
\end{figure}

\begin{table}[H]
	\centering
	\setlength{\tabcolsep}{4pt}
	\caption{
		Illustration how similarity band statistics are a driver of the trade-off.
        The MassSpecGym test set is partitioned into two halves based on either the $\sigma_{\max} - \sigma_{\min}$ or $\sigma_{\max}$ value of their candidate set.
        Then, the difference in HR@1 between the Contrastive (Embed-Cos) and IoU model (both using the MLP spectrum embedder model) is examined within each half.
        It is shown that regret is higher when similarity bands are wider.
	}
	\resizebox{0.65\linewidth}{!}{
	\begin{tabular}{lll}
		\toprule
        \textbf{Partition by} $\sigma_{\max} - \sigma_{\min}$ \\ \midrule
		Half & Mean $\sigma_{\max} - \sigma_{\min}$ & HR@1 Gap (Contr. - IoU) \\ \midrule
        Bottom & 0.21 & 5.63 \\
        Top & 0.67 & 7.15 \\ \midrule
        \textbf{Partition by} $\sigma_{\max}$ \\ \midrule
        Half & Mean $\sigma_{\max}$ & HR@1 Gap (Contr. - IoU) \\ \midrule
        Bottom & 0.24 & 5.71 \\
        Top & 0.70 & 7.07
        \\ \bottomrule
	\end{tabular}
	} \label{tab:partition}
\end{table}

\begin{figure}[H]
	\centering
	\includegraphics[width=0.95\linewidth]{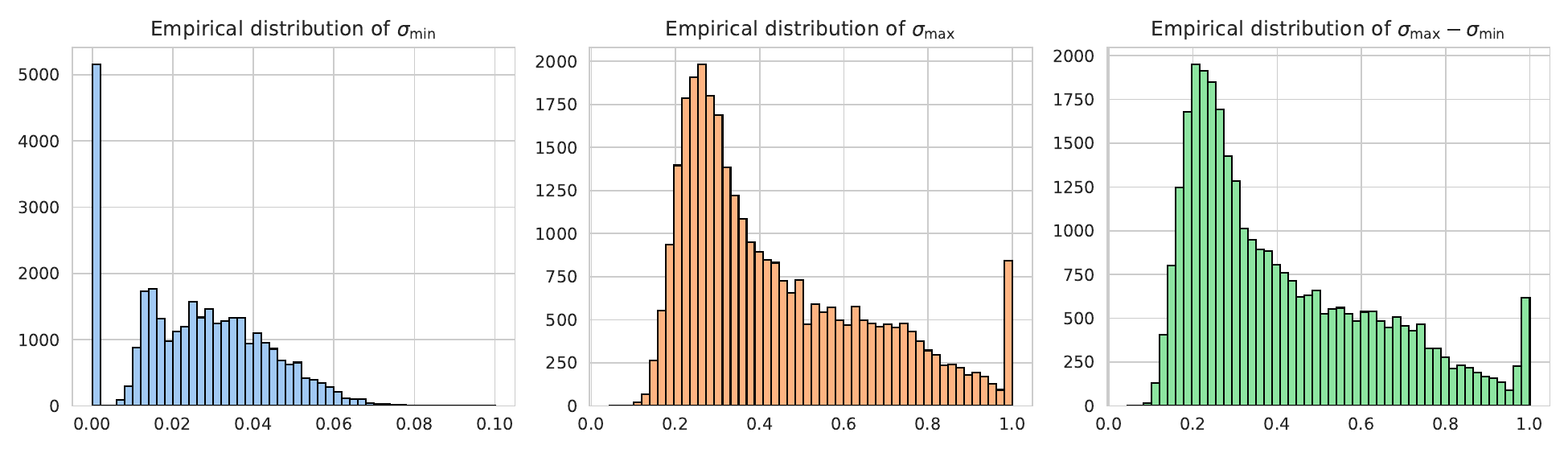}
	\caption{
		Empirical distribution of $\sigma_{\min}$, $\sigma_{\max}$, and $\sigma_{\max} - \sigma_{\min}$ 
        in all MassSpecGym equal-mass candidate sets using the 
        \textbf{Tanimoto similarity} on Morgan Fingerprints (4096 bits, radius 2) to compare candidates to true molecules.
        $\sigma_{\min}$ denotes the lowest similarity to the true molecule found in the candidate set.
        $\sigma_{\max}$ denotes the highest similarity to the true molecule found in the candidate set.
        In some cases, there are candidates with zero similarity to the true molecule (no fingerprint bits in common).
        Conversely, in some cases, there exist candidates with an equal fingerprint vector representation (i.e., $\sigma_{\max} = 1$ in some cases).
        This arises due to the problem of duplicate fingerprints.
        The median $\sigma_{\min}$ equals $0.026$, while the median $\sigma_{\max}$ equals $0.367$.
        }
	\label{fig:sigmas}
\end{figure}

\begin{figure}[H]
	\centering
	\includegraphics[width=0.95\linewidth]{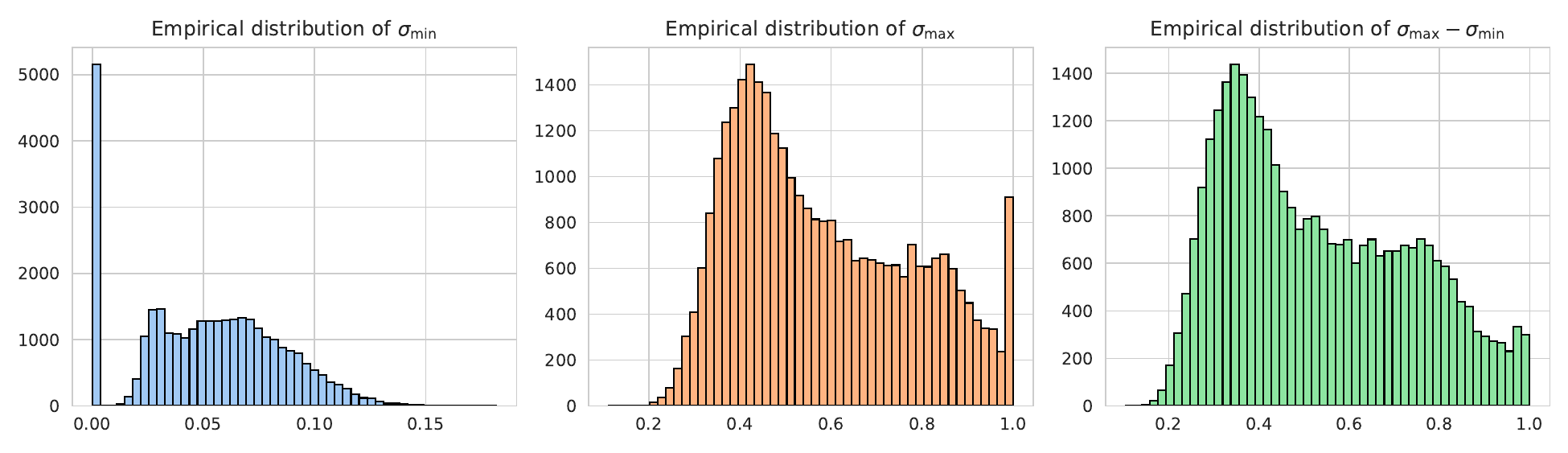}
	\caption{
		Empirical distribution of $\sigma_{\min}$, $\sigma_{\max}$, and $\sigma_{\max} - \sigma_{\min}$ 
        in all MassSpecGym equal-mass candidate sets using the \textbf{Cosine similarity} 
        on Morgan Fingerprints (4096 bits, radius 2) to compare candidates to true molecules.
        $\sigma_{\min}$ denotes the lowest similarity to the true molecule found in the candidate set.
        $\sigma_{\max}$ denotes the highest similarity to the true molecule found in the candidate set.
        In some cases, there are candidates with zero similarity to the true molecule (no fingerprint bits in common).
        Conversely, in some cases, there exist candidates with an equal fingerprint vector representation (i.e., $\sigma_{\max} =1$ in some cases).
        This arises due to the problem of duplicate fingerprints.
        The median $\sigma_{\min}$ equals $0.053$, while the median $\sigma_{\max}$ equals $0.538$.
	}
	\label{fig:cosine_sigmas}
\end{figure}

\begin{table}[H]
    \centering
    \setlength{\tabcolsep}{4pt}
    \caption{
        Empirical medians of $\sigma_\text{min}$, $\sigma_\text{max} $, and $\sigma_\text{max}-\sigma_\text{min}$ distributions in MassSpecGym equal-mass candidate sets using Tanimoto similarity, depending on the type of fingerprint.
    }
    \resizebox{0.4\linewidth}{!}{
        \begin{tabular}{lccc}
            \toprule
            Fingerprint type & $\sigma_\text{min}$ & $\sigma_\text{max} $& $\sigma_\text{max}-\sigma_\text{min}$\\ \midrule
            Morgan $r=2$  & 0.026 & 0.367 & 0.342 \\
            Morgan $r=4$  & 0.018 & 0.241 & 0.223 \\
            Morgan $r=6$  & 0.017 & 0.211 & 0.193 \\
            Morgan $r=8$  & 0.017 & 0.204 & 0.187 \\
            RDKit         & 0.039 & 0.512 & 0.471 \\
            MAP4          & 0.025 & 0.198 & 0.173 \\
            \bottomrule
        \end{tabular}
    }
    \vspace{-0.2cm}
	\label{tab:sigma_distros}
\end{table}

\section{The effect of contrastive negative candidate set selection} \label{sec:supp_hardnegs}

In the main results, contrastive loss function models are trained with the candidate sets provided by MassSpecGym~\citep{bushuiev2024massspecgym}.
MassSpecGym provides different candidate sets for the equal mass and equal formula candidate scenarios, both capped to a maximum of 256 candidates per input spectrum.
In our study, we evaluate each setting by training contrastive models using their respective candidate sets.
In this section, we briefly explore how training on different candidate sets influences performance results.
All experiments presented in this section are performed using the "Emb-Cos" contrastive model, as this is the one performing best overall.

As a first example, let us investigate how MLP models trained using one candidate set perform when evaluated on the other type of candidate set.
In other words, this experiment investigates if a model trained on one candidate set can generalize its retrieval towards other retrieval settings.
Here, we take the models trained on equal mass candidates and evaluate their retrieval on equal formula candidates, and vice-versa.
From \cref{tab:hard_cand}, it can be seen that models trained on equal mass candidates performs better on equal formula candidates than the models trained on equal formula candidates themselves.
This can potentially be an effect of the average candidate set being bigger using equal mass candidates.
As a consequence, it is possible that equal mass candidate sets contain more "hard" or "representative" negative examples through which a more discriminative model can be trained.

\begin{table}[h!]
    \centering
    \setlength{\tabcolsep}{4pt}
    \caption{
        Molecular retrieval and fingerprint accuracy performance on MassSpecGym using MLP models~\citep{bushuiev2024massspecgym}.
        The best and second-best performing models for each metric are indicated with boldface and underlined, respectively.
        For the models trained in this study, numbers report averages and standard deviations over 5 model runs.
    }
    \resizebox{1.0\linewidth}{!}{
        \begin{tabular}{llcccccc}
            \toprule
            &  &\multicolumn{3}{c}{Retrieval w. equal mass candidates}                                      & \multicolumn{3}{c}{Retrieval w. equal formula candidates}                                               \\ \cmidrule(lr){3-5} \cmidrule(lr){6-8}
            Candidate set origin & Max. (aver.) \# cands & \multicolumn{1}{c}{HR@1} & \multicolumn{1}{c}{HR@5} & \multicolumn{1}{c}{HR@20} & \multicolumn{1}{c}{HR@1} & \multicolumn{1}{c}{HR@5} & \multicolumn{1}{c}{HR@20} \\ \midrule
            MassSpecGym equal mass & 256 (252.8)       & $\mathbf{12.29 \pm 0.69}$ &	$\mathbf{26.09 \pm 1.43}$ &	$\mathbf{45.05 \pm 0.71}$     &            $\mathbf{14.17 \pm 0.39}$           & $\mathbf{28.36 \pm 0.46}$ & $\mathbf{48.86 \pm 0.94}$           \\
            MassSpecGym equal formula   & 256 (212.1)    & $10.43 \pm 0.66$ &  $21.91 \pm 1.07$ &  $39.08 \pm 0.72$  & $13.62 \pm 1.12$	& $26.33 \pm 1.06$ &	$46.50 \pm 0.94$         \\ \midrule
              "Hard" equal mass & 64 (63.9) & $07.74 \pm 0.61$	&$18.81 \pm 0.95$	&$36.4 \pm 1.62$ &$11.23 \pm 0.48$	&$24.23 \pm 0.94$	&$44.13 \pm 1.26$   \\
            "Hard" equal mass & 128 (127.7) & $08.13 \pm 0.57$	&$20.64 \pm 1.20$	&$39.81 \pm 1.70$& $12.25 \pm 0.91$	&$25.96 \pm 0.33$	&$45.73 \pm 0.62$ \\
            "Hard" equal mass & 256 (253.8) & $08.34 \pm 0.60$	&$21.51 \pm 1.64$	&$40.67 \pm 1.58$& $12.30 \pm 0.69$	&$25.86 \pm 1.40$	&$44.90 \pm 1.24$\\
            "Hard" equal mass & 512 (498.2) & $10.81 \pm 1.05$	&$23.37 \pm 0.96$	&$41.79 \pm 1.33$ &$12.38 \pm 0.68$	&$26.90 \pm 1.02$	& $45.88 \pm 1.17$  \\ 
            "Hard" equal mass & 1024 (959.5) &  $\underline{11.64 \pm 0.57}$	&$\underline{23.99 \pm 1.38}$	&$\underline{41.89 \pm 2.23}$ &  $\underline{14.10 \pm 0.60}$	&$\underline{27.55 \pm 0.85}$	& $\underline{47.15 \pm 0.87}$ \\ \midrule
            "Hard" equal formula & 64 (60.2) & $06.64 \pm 0.58$	&$17.06 \pm 1.76$	& $34.69 \pm 0.90$ &$11.47 \pm 0.36$	&$24.66 \pm 0.36$	&$43.60 \pm 0.24$          \\
            "Hard" equal formula & 128 (115.1) &$08.03 \pm 0.47$	 &$18.49 \pm 0.95$	&$34.82 \pm 1.08$ & $12.08 \pm 0.89$	&$25.24 \pm 0.57$	&$45.04 \pm 1.22$       \\
            "Hard" equal formula & 256 (214.3) & $07.66 \pm 0.38$	&$19.34 \pm 0.80$	&$37.74 \pm 1.36$ & $12.99 \pm 0.85$	&$26.46 \pm 0.59$	&$45.49 \pm 1.02$ \\
            "Hard" equal formula & 512 (385.8) &  $07.71 \pm 0.51$	&$19.95 \pm 1.73$	&$37.52 \pm 1.60$& $12.39 \pm 0.97$	&$25.64 \pm 0.90$	&$45.38 \pm 0.34$     \\
            "Hard" equal formula & 1024 (667.9) &   $08.68 \pm 1.59$	&$20.80 \pm 1.70$	&$38.91 \pm 1.12$  & $12.55 \pm 0.46$	&$27.08 \pm 0.95$	&$46.33 \pm 1.59$  \\ 
           \bottomrule
        \end{tabular}
    } \label{tab:hard_cand}
\end{table}

To further experiment with the notion of creating informative contrastive candidate sets, we create our own negative candidate sets for training models.
Note that, in this case, models are only trained using these custom negative compound sets, and are still evaluated on the MassSpecGym candidates in the end.
For this purpose, for every unique compound in MassSpecGym, we enumerate all compounds with (1) equal formula and (2) equal mass as retrieved in the PubChem 118M compound set included with MassSpecGym.
For every retrieved set, we subsequently sort the candidates by decreasing Tanimoto similarity (based on 4096 bit Morgan fingerprints), and only keep the $n$ first candidates.
In this way, candidate selection is performed to include the negative compounds which are "hardest" to distinguish by the contrastive model.
All models in this experiment were identically trained to those in the main experiments.
The only small addition is that for every max number of candidates $n\in\{64,128,256,512,1024\}$, the temperature was additionally re-tuned (identically as before in the main experiments) between values $\tau\in\{1/256, 1/64, 1/16\}$.
We perform this step as varying the number of candidates in the softmax during training may have an impact on optimal temperature values.
From \cref{tab:hard_cand}, it can be seen that there is an upward trend in performance when training models with an increasing number of negative candidates per spectrum.
Still, performance of models using hard negatives mined from PubChem do not exceed performance in comparison with the MassSpecGym equal mass negative sets, showcasing that also including broad chemical space from random PubChem selection benefits the final model.

\newpage
\section*{NeurIPS Paper Checklist}

\begin{enumerate}

\item {\bf Claims}
    \item[] Question: Do the main claims made in the abstract and introduction accurately reflect the paper's contributions and scope?
    \item[] Answer: \answerYes{}
    \item[] Justification: The abstract and Introduction state the paper's scope: a theoretical analysis of objective mismatch in fingerprint-based molecular retrieval and empirical validation on MassSpecGym and Enveda-180. The claims are supported by the regret analysis in \cref{sec:contrastive_as_retrieval,sec:regret_analysis} and by the experiments in \cref{sec:msmolemainres}.

\item {\bf Limitations}
    \item[] Question: Does the paper discuss the limitations of the work performed by the authors?
    \item[] Answer: \answerYes{}
    \item[] Justification: The Discussion notes limits in scope, including that empirical verification beyond molecular retrieval is left for future work, and that the practical behavior depends on uncertainty and fingerprint predictability. The paper also discusses limitations of the formal link between contrastive loss and HR@$k$ in \cref{sec:contrastiveHR}.

\item {\bf Theory assumptions and proofs}
    \item[] Question: For each theoretical result, does the paper provide the full set of assumptions and a complete (and correct) proof?
    \item[] Answer: \answerYes{}
    \item[] Justification: The main assumptions for the Bayes-risk and regret analyses are stated in \cref{sec:contrastive_as_retrieval,sec:regret_analysis}. The proofs for the theoretical statements are provided in the appendix, including \cref{app:Bayes-cosine,app:regret_proofs}.

    \item {\bf Experimental result reproducibility}
    \item[] Question: Does the paper fully disclose all the information needed to reproduce the main experimental results of the paper to the extent that it affects the main claims and/or conclusions of the paper (regardless of whether the code and data are provided or not)?
    \item[] Answer: \answerYes{}
    \item[] Justification: The experiments specify the MassSpecGym and Enveda-180 dataset, data splits, candidate settings, fingerprint representations, losses, model families, evaluation metrics, and hyperparameter tuning procedure in \cref{sec:msmolemainres,sec:suppmethodsmsmole}. The main tables report averages and standard deviations over five runs.

\item {\bf Open access to data and code}
    \item[] Question: Does the paper provide open access to the data and code, with sufficient instructions to faithfully reproduce the main experimental results, as described in supplemental material?
    \item[] Answer: \answerYes{} 
    \item[] Justification: The appendices list of hyperparameters in full detail. A zip file with code is co-submitted with the paper.

\item {\bf Experimental setting/details}
    \item[] Question: Does the paper specify all the training and test details (e.g., data splits, hyperparameters, how they were chosen, type of optimizer) necessary to understand the results?
    \item[] Answer: \answerYes{}
    \item[] Justification: \cref{sec:msmolemainres} describes the MassSpecGym and Enveda-180 splits, candidate sets, fingerprint targets, losses, and evaluation metrics. \cref{sec:suppmethodsmsmole} gives the architectures, optimizer, batch size, epoch budget, checkpointing procedure, learning-rate grid, and tuned hyperparameters.

\item {\bf Experiment statistical significance}
    \item[] Question: Does the paper report error bars suitably and correctly defined or other appropriate information about the statistical significance of the experiments?
    \item[] Answer: \answerYes{}
    \item[] Justification: The main experimental tables report averages and standard deviations over five model runs, as stated in the captions of \cref{tab:msmole_maintabv2,tab:fp_selection}. The appendix reports the same form of variability for the relevant tuning and ablation tables.

\item {\bf Experiments compute resources}
    \item[] Question: For each experiment, does the paper provide sufficient information on the computer resources (type of compute workers, memory, time of execution) needed to reproduce the experiments?
    \item[] Answer: \answerYes{} 
    \item[] Justification: We specify the type of GPU and the compute time per model trained.
    
\item {\bf Code of ethics}
    \item[] Question: Does the research conducted in the paper conform, in every respect, with the NeurIPS Code of Ethics \url{https://neurips.cc/public/EthicsGuidelines}?
    \item[] Answer: \answerYes{}
    \item[] Justification: The work analyzes molecular retrieval methods using a public scientific benchmark and does not involve human subjects, private data, or deployment decisions. No aspect of the described research appears to require deviation from the NeurIPS Code of Ethics.

\item {\bf Broader impacts}
    \item[] Question: Does the paper discuss both potential positive societal impacts and negative societal impacts of the work performed?
    \item[] Answer: \answerNA{} 
    \item[] Justification: The work is of mostly theoretical nature applied to the metabolomics domain. Our contributions are methodological and do not raise concerns beyond those well established when advancing computational methods for chemistry.
    
\item {\bf Safeguards}
    \item[] Question: Does the paper describe safeguards that have been put in place for responsible release of data or models that have a high risk for misuse (e.g., pre-trained language models, image generators, or scraped datasets)?
    \item[] Answer: \answerNA{}
    \item[] Justification: The paper does not release high-risk pretrained generative models or scraped datasets. The experiments use molecular mass-spectrometry benchmark data and model variants for scientific retrieval evaluation.

\item {\bf Licenses for existing assets}
    \item[] Question: Are the creators or original owners of assets (e.g., code, data, models), used in the paper, properly credited and are the license and terms of use explicitly mentioned and properly respected?
    \item[] Answer: \answerYes{} 
    \item[] Justification: Our code primarily uses MassSpecGym and Enveda-180 and respects its licenses. 

\item {\bf New assets}
    \item[] Question: Are new assets introduced in the paper well documented and is the documentation provided alongside the assets?
    \item[] Answer: \answerNA{}
    \item[] Justification: The manuscript does not currently does not release new assets.

\item {\bf Crowdsourcing and research with human subjects}
    \item[] Question: For crowdsourcing experiments and research with human subjects, does the paper include the full text of instructions given to participants and screenshots, if applicable, as well as details about compensation (if any)? 
    \item[] Answer: \answerNA{}
    \item[] Justification: The paper does not involve crowdsourcing experiments or research with human subjects.

\item {\bf Institutional review board (IRB) approvals or equivalent for research with human subjects}
    \item[] Question: Does the paper describe potential risks incurred by study participants, whether such risks were disclosed to the subjects, and whether Institutional Review Board (IRB) approvals (or an equivalent approval/review based on the requirements of your country or institution) were obtained?
    \item[] Answer: \answerNA{}
    \item[] Justification: The paper does not involve crowdsourcing or research with human subjects, so IRB approval or equivalent review is not applicable.

\item {\bf Declaration of LLM usage}
    \item[] Question: Does the paper describe the usage of LLMs if it is an important, original, or non-standard component of the core methods in this research? Note that if the LLM is used only for writing, editing, or formatting purposes and does \emph{not} impact the core methodology, scientific rigor, or originality of the research, declaration is not required.
    \item[] Answer: \answerNA{}
    \item[] Justification: The core methodology does not involve LLMs as an important, original, or non-standard component.

\end{enumerate}

\end{document}